\documentclass[11pt]{article}

\usepackage{fullpage}
\usepackage{algorithmic}
\usepackage{algorithm}
\usepackage{amssymb}
\usepackage{amsthm}
\usepackage{amsfonts,amsmath,amssymb,epsfig,color,float,graphicx,verbatim}
\usepackage{hyperref}
\usepackage{caption}
\usepackage{enumitem}
\usepackage{subcaption}

\renewcommand{\vec}[1]{\mathbf{#1}}
\newcommand{\norm}[2]{\left\|#1\right\|_{#2}}
\newcommand{\halfnorm}[1]{\norm{#1}{\frac{1}{2}}}
\newcommand{\expectation}[1]{\mathbb{E}\left[#1\right]}
\newcommand{\expectationA}[1]{\mathbb{E}_{A}\left[#1\right]}
\newcommand{\expectationAphase}[2]{\mathbb{E}_{A_{#1}}\left[#2\right]}
\newcommand{\expectationD}[1]{\mathbb{E}_{D}\left[#1\right]}
\newcommand{\expectationDA}[1]{\mathbb{E}_{D,A}\left[#1\right]}
\newcommand{\expectationDAphase}[2]{\mathbb{E}_{D,A_{#1}}\left[#2\right]}
\newcommand{\BigO}[1]{\ensuremath{\operatorname{O}\left(#1\right)}}

\newcommand{\BigOmega}[1]{\ensuremath{\operatorname{\Omega}\left(#1\right)}}
\newcommand{\clip}{\operatorname{clip}}

\newtheorem{theorem}{Theorem}[section]
\newtheorem{lemma}[theorem]{Lemma}

\title{Attribute Efficient Linear Regression
with Data-Dependent Sampling}
\author{Doron Kukliasnky\\ Weizmann Institute of Science\\\texttt{doronk@weizmann.ac.il}
\and Ohad Shamir\\ Weizmann Institute of Science\\\texttt{ohad.shamir@weizmann.ac.il}}
\date{}

\begin{document}

\maketitle

\begin{abstract}
In this paper we analyze a budgeted learning setting, in which the learner can only choose and observe a small subset of the attributes of each training example. We develop efficient algorithms for ridge and lasso linear regression, which utilize the geometry of the data by a novel data-dependent sampling scheme. When the learner has prior knowledge on
the second moments of the attributes, the optimal sampling probabilities can be calculated precisely, and result in data-dependent improvements factors for the excess risk over the state-of-the-art that may be as large as $O(\sqrt{d})$, where $d$ is the problem's dimension. Moreover, under reasonable assumptions our algorithms can use \emph{less} attributes than full-information
algorithms, which is the main concern in budgeted learning settings. To the best of our knowledge, these are the first algorithms able to do so in our setting. Where no such prior knowledge is available, we develop a simple estimation technique that given a sufficient amount of training examples, achieves similar improvements. We complement our theoretical analysis with experiments on several data sets which support our claims.
\end{abstract}

\section{Introduction}
Linear regression is a longstanding and effective method for learning a prediction model from various data sets. However, in
some scenarios, it cannot be utilized as-is: The algorithms that solve
this problem assume they can access all the attributes of the training set,
whereas there are some real-life scenarios where  the learner can access only a small number of attributes per training example. 

Consider, for example, the problem of medical diagnosis in which the learner wishes to determine whether a patient has some disease based on a series of medical tests. In order to build a linear model, the learner has to gather a set of
volunteers, perform diagnostic tests on them and use the tests results
as features. However, some of the volunteers may be reluctant to undergo
a large number of tests, as medical tests may cause physical discomfort, and
will prefer to undergo only a small number of them. During test time, however, patients are more likely to agree to undergo all  tests, to find a diagnosis
to their illness.

 Another example is the case where there is some cost associated with each attribute, whether computational or financial. For example, the \url{http://intelligence.towerdata.com}
web site allows users to buy marketing data about email addresses and pay per feature. The learner would like to minimize the cost, which is not necessarily  the number of examples.

This problem is known as budgeted learning \cite{greiner} or learning with limited
attribute observation (LAO) \cite{ben}. Formally, we use the local budget setting presented in \cite{ohad}: For each training example (composed of a $d$-dimensional attribute vector $\vec{x}$ and a target value $y$), we have a budget of $k+1$ attributes, where $k\ll d$, and we are able to choose which attribute we wish to reveal. This is different from the missing data setting, in which the selected attributes are given to us, and we are not able to choose which attributes to reveal, and from the feature selection setting, in which the output model includes only a subset of the attributes. In our setting, the goal is to find a good predictor despite the partial information at training time where a good predictor is defined as one that
minimizes the expected discrepancy between the predicted value, $\hat{y}$, and the target value, $y$. This discrepancy is generally measured by some kind of loss function, and we focus on the squared loss i.e. $\ell\left(\hat{y},y\right)=\frac{1}{2}\left(\hat{y}-y\right)^2$.
The expected discrepancy over the training set is called the risk.

We consider learning with respect to linear predictors, parameterized by a vector $\vec{w}\in \mathbb{R}^d$. Given an unlabeled example $\vec{x}$ (a vector of attributes), the prediction is defined as $\left<w,x\right>$\footnote{We ignore the bias term here, but this can be easily  handled by adding a constant dummy attribute that will always be revealed.}. In particular, we focus on two standard types of linear prediction problems: Those with $L_2$ bounded norm, which are the ridge regression scenario; and those with $L_1$ bounded norm, which are the lasso regression scenario.

Our basic approach is similar to the one proposed in \cite{hazan, ohad}, which uses online gradient descent with stochastic gradient estimates. The general idea behind it is to scan through the training set, calculate an unbiased gradient estimator based on each example (using only a small number of attributes), and plug it into a stochastic gradient descent method, thus minimizing the loss over the training set.

The algorithms in \cite{hazan,ohad} build the unbiased estimator using uniform sampling from the attributes of the example, eventually leading to ridge
algorithms with expected excess risk bounds of $\BigO{\sqrt{d/km}}$ after $m$ examples, compared with $\BigO{\sqrt{1/m}}$ for the online full-information algorithms that can view all the attributes \cite{kakade}, and lasso algorithms with an
additional $\log d$ factor for both settings (see Table~\ref{tbl:results}). Another interpretation of these results is that when viewing only $k$ out of $d$ attributes, the algorithms need $\BigO{d/k}$ times as many examples to obtain the same accuracy, thus examining the same number of attributes.
\cite{hazan} also provides a lower bound for the ridge scenario establishing that the ridge bound is not improvable in general.

In this paper, despite these seemingly unimprovable results, we show that they can in fact be improved. We do this by developing a novel sampling scheme which samples the attributes in a data-dependent manner: We sample attributes with large second moments more  than others, thus are able to gain a \emph{data-dependent} improvement factor. In other words, our sampling methods take advantage
of the geometry of the data distribution, and utilize it to extract more 'information'
out of each sample. Under reasonable assumptions, our methods
need to examine \emph{less} attributes to reach the same accuracy than the online full-information algorithms, thus optimizing the principal goal in
budgeted scenarios. To the best of our knowledge, ours are the first methods able to do so in the local budget setting.

We begin by assuming  prior knowledge of  the second moments of the data, namely $\expectationD{\vec{x}_i^2}$ for $i\in\left[d\right]$, where we use $\expectationD{\cdot}$ to denote the expectation with respect the data distribution.
Our risk bounds, under the assumptions of $\norm{\vec{x}}{2}\leq1$ in the ridge scenario and $\norm{\vec{x}}{\infty}\leq1$ in the lasso scenario are also summarized in Table~\ref{tbl:results}. To clarify the notation, $\halfnorm{\expectationD{\vec{x}^{2}}}$ is defined
as $\left(\sum_{i=1}^d\sqrt{\expectationD{\vec{x}_i^2}}\right)^2$, and 
$\norm{\expectationD{\vec{x}^{2}}}{1}$ is defined as $\sum_{i=1}^d\expectationD{\vec{x}_i^2}$.

\begin{table}[h]
\resizebox{\textwidth}{!}{%
\begin{tabular}{|c|c|c|c|}
\hline
                 & New Bound                                                                    & Old Bound                  & Online Full-Information Bound \\ \hline
Ridge Regression & $\BigO{\sqrt{\left(\halfnorm{\expectationD{\vec{x}^{2}}}+k\right)/km}}$      & $\BigO{\sqrt{d/km}}$       & $\BigO{\sqrt{1/m}}$           \\ \hline
Lasso Regression & $\BigO{\sqrt{\left(\norm{\expectationD{\vec{x}^{2}}}{1}+k\right)\log d/km}}$ & $\BigO{\sqrt{d\log d/km}}$ & $\BigO{\sqrt{\log d/m}}$           \\ \hline
\end{tabular}
}
\caption{The expected excess risk bounds of the various algorithms under the assumptions of $\norm{\vec{x}}{2}\leq1$ in the ridge scenario and $\norm{\vec{x}}{\infty}\leq1$ in the lasso scenario.}
\label{tbl:results}
\end{table}

It can be easily shown that both  $\halfnorm{\expectationD{\vec{x}^{2}}}$ and $\norm{\expectationD{\vec{x}^{2}}}{1}$ are smaller than or equal to $d$,  which proves that our bounds are always as good as the
previous bounds. In fact, the equalities hold only when all the moments are exactly the same. Otherwise, both values are strictly smaller than $d$,   making our bounds better
than the previous. This improvement factor is data-dependent and may be as large as $\BigO{\sqrt{d}}$
as both values can be small as $\BigO{1}$, when the moments decay at a sufficient rate. In fact, similar distributional assumptions are made in other successful algorithmic approaches such as AdaGrad (we further elaborate on the connection between our work and AdaGrad in Appendix \ref{app:adagrad}). When the attribute budget satisfies $k=\BigOmega{\halfnorm{\expectationD{\vec{x}^{2}}}}$ (or $k=\BigOmega{\norm{\expectationD{\vec{x}^{2}}}{1}}$ in the lasso scenario)
our bounds also coincide with the online full-information scenario.

Of course, a practical limitation of our approach is that the  second moments of the data may not be known in advance or easily computable in our attribute efficient setting. To address this, we split our algorithms into two phases: In the first phase, we use a simple yet effective estimation scheme that estimates the second moments of the attributes.
In the second phase, we use the same sampling scheme but with smoothed probabilities,
to compensate for the stochastic nature of the estimation phase. We prove that this method is always as good as the previous algorithms (up to constant factors) and
given sufficient training examples, achieves the same bounds as our algorithms with prior knowledge on the second moments of the attributes (up to constant factors).

The rest of this paper is organized as follows: In section~\ref{sec:pre}
we provide necessary background. In section~\ref{sec:ridge} we describe the existing state of the art algorithms for  attribute efficient ridge
regression, and develop our sampling scheme for the case where we have prior knowledge of the second moments of the attributes. We also develop an estimation
scheme for the case where we do not assume any prior knowledge of the second moments of the attributes, and present two variants
of the algorithm: one that does assume prior knowledge of $\halfnorm{\expectationD{\vec{x}^{2}}}$
only, and one that does not assume any prior knowledge at all. These two variants have
the same expected risk bounds (up to a constant factor), but differ by
the number of training examples needed in the estimation phase.
In section~\ref{sec:lasso} we provide similar results, this time for attribute
efficient lasso regression. When no prior knowledge of the second moment
of the attributes is available, the lasso scenario is simpler than the ridge scenario,
as prior knowledge of $\norm{\expectationD{\vec{x}^{2}}}{1}$ does not improve
the results.  In section~\ref{sec:experiments} we show experimental
results that support our theoretical claims, both on simulated and on well
known data sets. We finish with a summary in section~\ref{sec:discussion},
and short discussion about the connection between the AdaGrad method and our sampling scheme in appendix~\ref{app:adagrad}. 
\section{Preliminaries}\label{sec:pre}
\subsection{Notation}
Throughout this paper we use the following notations: We indicate scalars by a small letter, $a$, and vectors by a bold font, $\vec{a}$. We use $\vec{a}^2$ to indicate the vector for which $\vec{a}^2\left[i\right]=a\left[i\right]^2$ for all $i$, and $\vec{a}+b$ to indicate the vector for which $\left(\vec{a}+b\right)\left[i\right]=a\left[i\right]+b$. We  denote the $i$-th vector of the standard basis by $\vec{e}_i$. All our vectors lie in $\mathbb{R}^d$, where $d$ is the dimension. We indicate  the set of indices $1,..,n$ by $\left[n\right]$. We use $\norm{\vec{a}}{p}$ to indicate the $p$-norm of the vector, equal to $\left(\sum_{i=1}^d\left|a_i\right|^p\right)^\frac{1}{p}$. We apply this notation also for the case where $p=\frac{1}{2}$ i.e. $\norm{\vec{a}}{\frac{1}{2}}=(\sum_{i=1}^d\sqrt{\left|a_i\right|})^2$,  even though this is not a proper norm, as the triangle inequality does not hold. We also use $\norm{\vec{a}}{\infty}$ to indicate the infinity norm, $\max_i\left|a_i\right|$. We use $\left<\vec{a},\vec{b}\right>$ to indicate the standard inner product, $\sum_{i=1}^da_ib_i$.   
We  denote the expectation with respect to the randomness of the algorithm (attribute sampling) by
$\expectationA{\cdot}$, the expectation with respect to the data distribution by $\expectationD{\cdot}$ and the expectation with respect to both by $\expectationDA{\cdot}$. For the two-phased algorithms, we use $\expectationDAphase{i}{\cdot}$ where $i\in\left\{1,2\right\}$ to denote
the expectation with respect to the data distributions and the randomness of the algorithm during the $i$-th phase.

\subsection{Linear Regression}
The general framework for regression assumes the learner has a training set: 
$\left\{\left(\vec{x}_t,y_t\right)\in\mathbb{R}^{d}\times\mathbb{R}\right\}_{t=1}^m$,
where each $\vec{x}_t$ is  a data point, represented by a vector of attributes, and $y_t$ is the desired target value. The goal of the learner is to find a weight vector $\vec{w}$, such that $\hat{y_t}=\left<\vec{w},\vec{x}_t\right>$ is a good estimator of $y_t$, in the sense that it minimizes some penalty function over the entire data set. We focus on the most popular choice for
such a function -
the squared loss: $\ell\left(\hat{y},y\right)=\frac{1}{2}\left(\hat{y}-y\right)^2$. We denote the loss induced by $\left(\vec{x}_t,y_t\right)$ as $\ell_t\left(\vec{w}\right)$.

We follow the standard framework for
statistical learning \cite{haussler} and assume the training set was sampled i.i.d.
from some joint distribution $\mathcal{D}$. The goal of the learner is to find a predictor that minimizes the
risk, defined as the expected loss:
\begin{equation*}
L_{\mathcal{D}}\left(\vec{w}\right)=\mathbb{E}_{\left(\vec{x},y\right)\sim\mathcal{D}}\left[\ell\left(\vec{w}^{T}\vec{x},y\right)\right].
\end{equation*} 
Since the distribution $\mathcal{D}$ is unknown, the learner relies on the given training set, $S=\left\{\left(\vec{x}_1,y_1\right),..,\left(\vec{x}_m,y_m\right)\right\}$ that is assumed to be sampled i.i.d. from $\mathcal{D}$. 

Finding this minimum may result in over fitting the data,  therefore it is common to limit the size of the hypothesis class by adding some regularization constraint
on the norm of $\vec{w}$, requiring it to be smaller than or equal to some
value. The first of the two main scenarios of regression is ridge regression, where we have the 2-norm constraint, and the hypothesis class is $\mathcal{F}=\left\{\vec{w}|\norm{\vec{w}}{2}\leq B\right\}$.
If we assume $\left|y_t\right|\leq B$, using the Cauchy-Schwarz inequality, we can assume without loss of generality that $\norm{\vec{x}}{2}\leq1$ with
probability $1$. The second is lasso regression, where we have the 1-norm constraint, and the hypothesis class is $\mathcal{F}=\left\{\vec{w}|\norm{\vec{w}}{1}\leq B\right\}$.
Since we assume $\left|y_t\right|\leq B$, using the H{\"o}lder inequality, we can assume without loss of generality that $\norm{\vec{x}}{\infty}\leq1$
with probability $1$.

In the full-information regression scenario, the learner has access to all the attributes of $\vec{x}_{t}$, whereas in the attribute efficient scenario, the learner can sample at most $k+1$ attributes out of $d$ from each vector $\vec{x}_{t}$. 

\subsection{Related Work}
The scenario of learning with limited attribute access was first introduced by Ben-David \& Dichterman \cite{ben},
under the term "Learning with Restricted Focus of Attention". There are two
popular types of constraints: The first, which we address in this paper, is the local
budget constraint, where the learner has access to $k+1$ attributes per training
example. The second is the global budget constraint where the learner has
access to a total number of $K$ attributes, and may spread them freely among
all the training examples, as long as the total number of attributes seen
does not exceed $K$. Clearly, any upper bound for the local budget setting
holds also in the global budget setting for $K=\left(k+1\right)m$.

Cesa-Bianchi et al. in \cite{ohad} were the first to build an efficient linear algorithm for the local budget scenario, and asked the question of whether there exists an efficient algorithm for the attribute efficient scenario that can reach a similar accuracy as the full attribute scenario, while seeing $\BigO{md/k}$ examples and  from each example being able to sample only $\BigO{k}$ attributes.
Such a result would imply that in the attribute efficient scenario, we can learn just as well as in the full-information scenario, after seeing the same number of attributes ($\BigO{m d}$ in both cases). Thus, we can trade-off between the number of examples and the amount of information received on each example.
They also proved a lower sample complexity bound of $\BigOmega{d/k\epsilon}$ examples for learning an $\epsilon$-accurate linear regressor.

Later on, Hazan et al. \cite{hazan} showed that the answer is yes, up to global constants for both the ridge and lasso scenarios. Their approach for
ridge regression was based on the Online Gradient Descent method \cite{zink} and on the EG algorithm \cite{warmuth} for the lasso scenario.
In both cases, at each iteration, the learner uses an unbiased estimator of the gradient, and updates the current weight vector accordingly. The key idea is that by sampling just a few attributes using an appropriate scheme, the learner can still build an unbiased estimator of the gradient, even in the attribute efficient scenario,
and by expectation, perform a gradient step in the correct direction. They also complemented the ridge regression result by proving a corresponding lower sample complexity  bound of $\BigOmega{d/k\epsilon^2}$ examples for learning an $\epsilon$-accurate ridge regressor.

\section{Attribute Efficient Ridge Regression}\label{sec:ridge}
In this section we present our algorithms for  ridge regression, where the loss is the squared loss: $\ell\left(\vec{w};\vec{x}_t,y_t\right)=\frac{1}{2}\left(\left<\vec{w},\vec{x}_t\right>-y_t\right)^2$, and the 2-norm is bounded, $\norm{\vec{w}}{2}\leq B$. The generic approach to the ridge attribute efficient scenario, which we call the General Attribute Efficient Ridge Regression (GAERR) algorithm and is presented in Algorithm \ref{alg:gaerr}, was first developed in \cite{ohad,hazan} and  is based on the
Online Gradient Descent (OGD) algorithm with gradient estimates. 

The OGD algorithm goes over the training set, and for each example builds an unbiased estimator of the gradient. 
Afterwards, the algorithm updates the current weight vector, $\vec{w}_t$, by performing a step of size $\eta$ in the opposite direction to the gradient estimator. The result is projected over the $L_2$ ball of size $B$, yielding $\vec{w}_{t+1}$. At the end, the algorithm outputs the average of all $\vec{w}_t$.
The algorithm converges to the global minimum, as the minimization problem is convex in $\vec{w}$.
 
The gradient of the squared loss is $\nabla\ell\left(\vec{w};\vec{x}_t,y_t\right)=\left(\left<\vec{w},\vec{x}_t\right>-y_t\right)\cdot\vec{x}_t$,
and the key idea of the GAERR algorithm is how to use the budgeted sampling to construct an
unbiased estimator for the gradient. The GAERR algorithm does so by sampling $k+1$ attributes out of the $d$ attributes of the sample where $k>0$ is the a budget parameter\footnote{As in the AERR algorithm, we assume we have a budget of at least $2$ attributes per training sample.}: First, it samples $k$ attributes with probabilities
$q_i$ and by weighting them correctly, builds an unbiased estimator for the data point $\widetilde{\vec{x}}_t$. Then it samples one attribute with probability $p_{j_{t}}=\frac{w_{t,j_t}^2}{\norm{\vec{w}_t}{2}^2}$
and by a simple calculation obtains an unbiased estimator of the inner product.
Reducing the label, $y_t$, yields the unbiased estimator, $\widetilde{\phi}_{t}$.
Finally, the algorithms multiplies the estimator of the inner product minus
the label, $\widetilde{\phi}_{t}$, by the estimator of the data point, $\widetilde{\vec{x}}_t$, thus 
building an unbiased estimator of the gradient for the point, $\widetilde{\vec{g}}_t$. 
\begin{algorithm}[h!]
  \caption{GAERR \newline 
  Parameters: $B,\eta>0$ and $q_i$ for $i\in\left[d\right]$}
  \label{alg:gaerr}
  \begin{algorithmic}[1]
  \REQUIRE training set $S=\left\{\left(\vec{x}_{t},y_{t}\right)\right\}_{t\in\left[m\right]}$ and $k>0$
  \ENSURE regressor $\bar{\vec{w}}$ with $\norm{\bar{\vec{w}}}{2}\leq B$
  \STATE Initialize $\vec{w}_{1}\neq 0$, $\norm{\vec{w}_{1}}{2}\leq B$ arbitrarily
 
  \FOR{$t=1$ to $m$}
    \FOR{$r=1$ to $k$}
    \STATE Pick $i_{t,r}\in \left[d\right]$ with probability $q_{i_{t,r}}$ and observe $\vec{x}_{t}\left[i_{t,r}\right]$
    \STATE $\widetilde{\vec{x}}_{t,r}\leftarrow\frac{1}{q_{i_{t,r}}}\vec{x}_{t}\left[i_{t,r}\right]  \vec{e}_{i_{t,r}}$

    \ENDFOR
  \STATE $\widetilde{\vec{x}}_{t} \leftarrow \frac{1}{k}\sum_{r=1}^{k}\widetilde{\vec{x}}_{t,r}$
  
  \STATE Choose $j_{t}\in\left[d\right]$ with probability $p_{j_{t}}=\frac{w_{t,j_t}^2}{\norm{\vec{w}_t}{2}^2}$
   and observe $\vec{x}_{t}\left[j_{t}\right]$
  \STATE $\widetilde{\phi}_{t} \leftarrow \frac{w_{t,j}}{p_{j_{t}}}\vec{x}_{t}\left[j_{t}\right]-y_t$
  \STATE $\widetilde{\vec{g}}_{t} \leftarrow \widetilde{\phi}_{t} \cdot \widetilde{\vec{x}}_{t}$
  \STATE $\vec{v}_{t} \leftarrow \vec{w}_{t}-\eta\widetilde{\vec{g}}_{t}$
  \STATE $\widetilde{\vec{w}}_{t+1} \leftarrow \vec{v}_{t} \cdot \frac{B}{\max\left\{{\left\|\vec{v}_t\right\|_{2},B}\right\}}$
  \ENDFOR
  \STATE $\bar{\vec{w}} \leftarrow \frac{1}{m}\sum_{t=1}^{m}\vec{w}_{t}$
  \end{algorithmic}
\end{algorithm}

\pagebreak
The expected risk bound of the GAERR algorithm is presented in the next theorem which is a slightly more general version of Theorem~3.1 in \cite{hazan}. 
\begin{theorem}\label{thm:maingaerr}
Assume the distribution $\mathcal{D}$ is such that $\norm{\vec{x}}{2}\leq1$ and $\left|y\right|<B$ with probability 1. Let $\bar{\vec{w}}$ be the output of GAERR when run with step size $\eta$ and let $\max_t\expectationDA{\norm{\widetilde{\vec{g}_t}}{2}^2}\leq
G^2$. Then for any $\vec{w}^{*}\in\mathbb{R}^{d}$ with $\norm{\vec{w}^{*}}{2}\leq
B$,
\begin{equation*}
\expectationDA{L_{\mathcal{D}}\left(\bar{\vec{w}}\right)}\leq
L_{\mathcal{D}}\left(\vec{w^{*}}\right)+\frac{2B^{2}}{\eta m} + \frac{\eta}{2}G^{2}.
\end{equation*}
\end{theorem}
The general idea of the proof is that $\widetilde{\vec{g}}_{t}$ is an unbiased estimator of the gradient, therefore we can use the standard analysis of the OGD algorithm. The full proof can be found in appendix \ref{app:thm:maingaerrproof}.

The AERR algorithm is one variant of the GAERR algorithm. It was presented in \cite{hazan} and uses uniform sampling to estimate $\vec{x}_t$. In our GAERR notation it uses
\begin{equation*}
q_i=\frac{1}{d}\:\:\:\forall i\in \left[d\right].
\end{equation*}
The authors prove (Lemma~3.3 in \cite{hazan}) that for the AERR algorithm, $G^2\leq8B^2d/k$, which together with Theorem~\ref{thm:maingaerr} and using
$\eta=\frac{2B}{G\sqrt{m}}$ yields an expected risk bound of $4B^2\sqrt{\frac{2d}{km}}$. They also prove that up to constant factors, their algorithm is optimal, by showing a corresponding lower bound.

This, however, is not the end of the story. By analyzing the bound, we show that we can improve the bound in a data-dependent manner. Theorem~\ref{thm:maingaerr} shows us that the expected risk bound is proportional to $G$, therefore we wish to develop a sampling method that minimizes the 2-norm of the gradient estimator.

 The gradient estimate consists of estimating the inner product and estimating $\vec{x}_t$. To estimate $\vec{x}_t$, we use the following procedure: we sample $k$ indices, $i_{t,r}$, from $1..d$ with probability $q_i$, and use 
\begin{equation}\label{eq:ridgeestimator}
\widetilde{\vec{x}}_t=\frac{1}{k}\sum_{r=1}^{k}\frac{1}{q_{i_{t,r}}}\vec{x_t}\left[i_{t,r}\right]\vec{e}_{i_{t,r}}
\end{equation}
 as an estimator for $\vec{x}_t$. The next lemma will assist in bounding its 2-norm.

\begin{lemma}\label{lm:twonormbound}
For every distribution $\left(q_1,..,q_d\right)$ where $q_i\geq0$ and $\sum_{i=1}^dq_i=1$, we have 
$\expectationDA{\norm{\widetilde{\vec{x}}_{t}}{2}^2}=\frac{1}{k}\expectationDA{\norm{\widetilde{\vec{x}}_{t,r}}{2}^2}+\frac{k-1}{k}\expectationD{\norm{\vec{x}}{2}}^2$.
\end{lemma}
The proof can be found in appendix \ref{app:twonormboundproof}.

Since
\begin{equation}\label{eq:xtrnorm}
\expectationDA{\norm{\widetilde{\vec{x}}_{t,r}}{2}^2}=
\expectationDA{\widetilde{\vec{x}}_{t,r}\left[i_{t,r}\right]^{2}}=
\sum_{i=1}^d\frac{1}{q_i}\expectationD{x_i^2},
\end{equation}
in order to minimize the 2-norm of the estimator, we need to solve the following optimization problem:

\begin{equation*}
\begin{aligned}
& \underset{q_i}{\text{minimize}}
& & \frac{1}{k}\sum_{i=1}^d\frac{1}{q_i}\expectationD{x_i^2}+\frac{k-1}{k}\expectationD{\norm{\vec{x}}{2}}^2 \\
& \text{subject to}
& & \sum_{i=1}^dq_i=1,\:\forall i\:q_i\geq0.
\end{aligned}
\end{equation*}
This problem  is equivalent to 
\begin{equation}
\begin{aligned}
& \underset{q_i}{\text{minimize}}
& & \sum_{i=1}^d\frac{1}{q_i}\expectationD{x_i^2} \\
& \text{subject to}
& & \sum_{i=1}^dq_i=1,\:\forall i\:q_i\geq0,
\end{aligned}
\end{equation}
which can easily be solved using the Lagrange multipliers method to yield the solution
\begin{equation}\label{eq:qis}
q_i=\frac{\sqrt{\expectationD{x_i^2}}}{\sum_{j=1}^d\sqrt{\expectationD{x_j^2}}}.
\end{equation}

We also use 
\begin{equation}
\widetilde{\phi}_{t} = \frac{w_{t,j}}{p_{j_{t}}}\vec{x}_{t}\left[j_{t}\right]-y_{t}
\end{equation}
as an estimator for the inner product minus the label. The next lemma will assist in bounding its 2-norm.

\begin{lemma}\label{lm:twonormboundphi}
Using our sampling method we have $\expectationDA{\widetilde{\phi_{t}}^2}\leq4B^2$.
\end{lemma}

The proof can be found in appendix \ref{app:twonormboundphiproof}.

\label{par:innerproductimprovement}
We could have followed a similar optimization strategy for finding the optimal sampling distribution for estimating the inner product. This strategy would have yielded that the optimal probabilities  are $p_i=\frac{\sqrt{\vec{w}_{t,i}^2\expectationD{x^2_i}}}{\sum_{j=1}^d\sqrt{\vec{w}_{t,j}^2\expectationD{x^2_j}}}$. We, however, were not able to prove the superiority of this sampling
method analytically and it was left out of the algorithm analysis.

Altogether, we  formulate a lemma that will bound the gradient estimate.
\begin{lemma}\label{lm:gaerrmain}
The GAERR algorithm generates gradient estimates that for all $t$, $\expectationDA{\left\|\widetilde{\vec{g}}_{t}\right\|_{2}^{2}} \leq 4B^{2}\left(\frac{1}{k}\expectationDA{\norm{\widetilde{\vec{x}}_{t,r}}{2}^{2}}+1\right)$.
\end{lemma}

\begin{proof}
This lemma follows directly from Lemmas \ref{lm:twonormbound} and \ref{lm:twonormboundphi}, using the independence of $\widetilde{\vec{x}}_{t}$ and $\widetilde{\phi_{t}}$ given $\vec{x}_t$ and $\norm{\vec{x}}{2}\leq1$.
\end{proof}

\subsection{Known Second Moment Scenario}
If we assume we have prior knowledge of the second moment of each attribute, namely $\expectationD{x_i^2}$ for all $i\in\left[d\right]$, we can use equation $\eqref{eq:qis}$ to calculate the optimal values of the $q_i$-s. This is the idea behind our DDAERR (Data-Dependent Attribute Efficient Ridge Regression) algorithm.

The expected risk bound of our algorithm is formulated in the next theorem.
\begin{theorem}\label{thm:mainridge}
Assume the distribution $\mathcal{D}$ is such that $\norm{\vec{x}}{2}\leq1$
and $\left|y\right|\leq B$ with probability 1 and $\expectationD{x_i^2}$ are known for $i\in\left[d\right]$.
Let $\bar{\vec{w}}$ be
the output of DDAERR, when run with $\eta=\frac{1}{\sqrt{m\left(\frac{1}{k}\halfnorm{\expectationD{\vec{x}^{2}}}+1\right)}}$.
Then for any $\vec{w}^{*}\in\mathbb{R}^{d}$ with $\left\|\vec{w}^{*}\right\|_{2}\leq
B$,
\begin{equation*}
\expectationDA{L_{\mathcal{D}}\left(\bar{\vec{w}}\right)}\leq
L_{\mathcal{D}}\left(\vec{w^{*}}\right)+4\frac{B^{2}}{\sqrt{m}}\sqrt{\frac{1}{k}\halfnorm{\expectationD{\vec{x}^{2}}}+1}.
\end{equation*}
\end{theorem}

\begin{proof}[Proof of Theorem~\ref{thm:mainridge}]
The theorem follows directly from Theorem~\ref{thm:maingaerr}, Lemma~\ref{lm:gaerrmain}, equation~\eqref{eq:xtrnorm} and the calculated $q_i$-s in equation \eqref{eq:qis}. \end{proof}

Recalling that with probability 1 we have $\norm{\vec{x}}{2}\leq1$, it is easy to see that $\halfnorm{\expectationD{\vec{x}^{2}}}\leq d$, therefore the DDAERR algorithm always performs at least as well as the AERR algorithm\footnote{If $\halfnorm{\expectationD{\vec{x}^{2}}}= d$ it is easy to see that $\expectationD{x_i}=\frac{1}{d}$ for all $i\in\left[d\right]$. In this case, all the $q_i$-s are equal to $\frac{1}{d}$ and the DDAERR and AERR algorithms coincide.}. However,  $\halfnorm{\expectationD{\vec{x}^{2}}}$ may also be much smaller than $d$, in cases where the second moments varies between attributes or the vector is sparse. In these cases, we may gain a significant improvement. For example, if we consider a polynomial attribute decay such as  $\expectationD{x_i^{2}}=\frac{i^{-2}}{\sum_{j=1}^dj^{-2}}$, we have $\halfnorm{\expectationD{\vec{x}^{2}}}=\BigO{\log^2d}$ which is significantly smaller than $d$.

\subsection{Unknown Second Moment Scenario, Known $\norm{\expectationD{\vec{x}^{2}}}{\frac{1}{2}}$}
The solution presented in the previous section requires exact knowledge of $\expectationD{x_i^2}$ for all $i$. Such prior knowledge may not be available
when the learner is faced with a new learning task. Thus, we turn to consider the case where the moments are initially unknown. We will still assume that the learner can guess or estimate the step size, which depends only on the scalar quantity $\norm{\expectationD{\vec{x}^{2}}}{\frac{1}{2}}$. In the next section, we will consider the case where even that information is unknown.

The problem in this scenario is that without prior knowledge of the second moments of the attributes, the learner cannot calculate the optimal $q_i$-s via equation~$\eqref{eq:qis}$. To address this issue we split the learning into two phases: In the first phase we run on the first $m_1$ training
 examples and estimate the second moments by sampling
the attributes uniformly at random. In the second phase we run on the next $m_2$ training examples, and perform the regular DDAERR algorithm, with
a slight modification - in the calculation of the $q_i$-s, we use an upper confidence interval instead of the second moments themselves.
We assume $m_2$ is on the order of $m$. This approach is the basis for our Two-Phased DDAERR algorithm (Algorithm~\ref{alg:onlineDDAERR}). The estimate for $\norm{\expectationD{\vec{x}^{2}}}{\frac{1}{2}}$ does not assist in the calculation of the $q_i$-s, but will give us the optimal step size, $\eta$.
\begin{algorithm}[h!]
  \caption{Two-Phased DDAERR \newline 
  Parameters: $m_1,m_2,\delta,B,\eta>0$}
  \label{alg:onlineDDAERR}
  \begin{algorithmic}[1]
  \REQUIRE training set $S=\left\{\left(\vec{x}_{t},y_{t}\right)\right\}_{t\in\left[m_{1}+m_{2}\right]}$ and $k>0$
  \ENSURE regressor $\bar{\vec{w}}$ with $\norm{\bar{\vec{w}}}{2}\leq B$
  \STATE Initialize $\vec{w}_{1}\neq 0$, $\norm{\vec{w}_{1}}{2}\leq B$ arbitrarily
  \STATE Initialize $\vec{A}$, $counts$ and $square\_sums$ - arrays of size $d$ with zeros
  \FOR{$t=1$ to $m_{1}$}
    \FOR{$r=1$ to $k+1$}
      \STATE Pick $i_{t,r}\in\left[d\right]$ uniformly at random
      \STATE $counts\left[i_{t,r}\right] \leftarrow counts\left[i_{t,r}\right]+1$
      \STATE $square\_sums\left[i_{t,r}\right] \leftarrow square\_sums\left[i_{t,r}\right]+\vec{x}_{t}\left[i_{t,r}\right]^{2}$
      \STATE $\vec{A}\left[i_{t,r}\right] \leftarrow \frac{square\_sums\left[i_{t,r}\right]}{counts\left[i_{t,r}\right]}$
  \ENDFOR
  \ENDFOR
  \STATE $\epsilon \leftarrow \frac{d\log{\frac{2d}{\delta}}}{\left(k+1\right)m_1}$
  \STATE Run GAERR with $q_i=\frac{\sqrt{\vec{A}\left[i\right]+\frac{13}{6}\epsilon}}{\sum_{j=1}^{d}\sqrt{\vec{A}\left[{j}\right]+\frac{13}{6}\epsilon}}$ on the following $m_2$ examples and return its output
  \end{algorithmic}
\end{algorithm}

Note that in practice, one can actually run the AERR algorithm during the first phase, in order to obtain a better starting point for the second phase. We ignore this improvement in our analysis below, but incorporate it in the experiments presented in section~\ref{sec:experiments}.

There are other variants of this type, the most apparent of them is to use the same samples that estimate the gradient to estimate the moments themselves. This method, even though in some cases may be superior to our method, will not yield better results in the worst case scenarios because we may never get accurate enough estimations for some of the attributes.

The expected risk bound of the algorithm is formulated in the following theorem.

\begin{theorem}\label{thm:onlineridgepartialfinal}
Assume the distribution $\mathcal{D}$ is such that $\norm{\vec{x}}{2}\leq1$
and $\left|y\right|\leq B$ with probability 1. Assume further that the value $\norm{\expectationD{\vec{x}^{2}}}{\frac{1}{2}}$ is known.
Let $\bar{\vec{w}}$ be the output of Two-Phased DDAERR when run with 
\begin{equation*}
\eta=\max\left(\sqrt{\frac{k}{6dm_{2}}},\sqrt{\frac{k}{m_2\left(2\norm{\expectationD{\vec{x}^{2}}}{\frac{1}{2}}+ 2\sqrt{\frac{5}{3}}d\sqrt{\norm{\expectationD{\vec{x}^{2}}}{\frac{1}{2}}}\sqrt{\frac{d\log{\frac{2d}{\delta}}}{\left(k+1\right)m_1}}+k\right)}}\right)\text{.}
\end{equation*}
Then for all $m_1$ and for any $\vec{w}^{*}\in\mathbb{R}^{d}$ with $\norm{\vec{w}^{*}}{2}\leq B$, with probability $1$ over the first phase,
we have
\begin{equation*}
\expectationDAphase{2}{L_{\mathcal{D}}\left(\bar{\vec{w}}\right)}-
L_{\mathcal{D}}\left(\vec{w}^{*}\right)\leq
\frac{4B^2}{\sqrt{m_2}}\sqrt{\frac{6d}{k}}. \end{equation*}
Also, with probability $\geq 1-\delta$ over the first phase,
we have\begin{equation*}
\expectationDAphase{2}{L_{\mathcal{D}}\left(\bar{\vec{w}}\right)}-
L_{\mathcal{D}}\left(\vec{w}^{*}\right)\leq\\
\frac{4B^2}{\sqrt{m_2}}\sqrt{\frac{2}{k}\norm{\expectationD{\vec{x}^{2}}}{\frac{1}{2}}+\frac{2}{k}\sqrt{\frac{5}{3}}d\sqrt{\norm{\expectationD{\vec{x}^{2}}}{\frac{1}{2}}}\sqrt{\frac{d\log{\frac{2d}{\delta}}}{\left(k+1\right)m_1}}+1}. \end{equation*}
\end{theorem}

With probability $1$ over the first phase, regardless of the value of $m_1$, the expected risk bound is at most $\BigO{\frac{B^2}{\sqrt{km_2}}\sqrt{d}}$, which is the same bound of the AERR algorithm. This means that the Two-Phased DDAERR algorithm performs with probability  $1$ over the first phase as well as the AERR algorithm, up to a constant factor. Second, as $m_1$ increases, the expected risk bound turns to $\BigO{\frac{B^2}{\sqrt{km_2}}\sqrt{\norm{\expectationD{\vec{x}^{2}}}{\frac{1}{2}}+d\sqrt{\norm{\expectationD{\vec{x}^{2}}}{\frac{1}{2}}}\sqrt{\frac{d\log{\frac{2d}{\delta}}}{\left(k+1\right)m_1}}+k}}$. Therefore, if $m_1\gg\frac{d\norm{\expectationD{\vec{x}^{2}}}{\frac{1}{2}}\log{\frac{2d}{\delta}}}{k+1}$, we achieve an improvement over the AERR algorithm. If $m_1\geq\frac{d^3\log{\frac{2d}{\delta}}}{\left(k+1\right)\norm{\expectationD{\vec{x}^{2}}}{\frac{1}{2}}}$, the bound becomes $\BigO{\frac{B^2}{\sqrt{km_2}}\sqrt{\norm{\expectationD{\vec{x}^{2}}}{\frac{1}{2}}+k}}$, which is the same bound as in the regular DDAERR algorithm which assumes prior knowledge of the second moment of the attributes.

The conclusion is that even if we do not have prior knowledge of
the second moments of the attributes, but can guess $\norm{\expectationD{\vec{x}^{2}}}{\frac{1}{2}}$, we still should prefer our Two-Phased DDAERR algorithm over the AERR algorithm.
In the next section, we analyze the case where even $\norm{\expectationD{\vec{x}^{2}}}{\frac{1}{2}}$ is unknown.

\subsubsection{Proof of Theorem~\ref{thm:onlineridgepartialfinal}}
The main goal of the proof is to bound the expected squared 2-norm of the gradient estimator from above. By using Lemma~\ref{lm:gaerrmain}, all that remains is to upper bound $\expectationDAphase{2}{\left\|\widetilde{\vec{x}}_{t,r}\right\|_{2}^{2}}$. In the next lemma we show two different upper bounds on $\expectationDAphase{2}{\left\|\widetilde{\vec{x}}_{t,r}\right\|_{2}^{2}}$. The first states that with probability $1$ over the first phase $\expectationDAphase{2}{\left\|\widetilde{\vec{x}}_{t,r}\right\|_{2}^{2}}\leq5d$, meaning that up to a constant factor the bound is the same as in the AERR algorithm. The second bound decreases in $\epsilon$, and will help up to analyze the convergence rate of the algorithm. 
\begin{lemma}\label{lm:ridgenotworse}
For all $m_1$ and $t>m_1$, with probability $1$ over the first phase, we have
\begin{equation*}
\expectationDAphase{2}{\left\|\widetilde{\vec{x}}_{t,r}\right\|_{2}^{2}}\leq5d,
\end{equation*}
and with probability $\geq 1-\delta$ over the first phase, we have
\begin{equation*}
\expectationDAphase{2}{\norm{\widetilde{\vec{x}}_{t,r}}{2}^{2}} \leq 2\norm{\expectationD{\vec{x}^{2}}}{\frac{1}{2}}+ 2\sqrt{\frac{5}{3}}d\sqrt{\norm{\expectationD{\vec{x}^{2}}}{\frac{1}{2}}}\sqrt{\epsilon}.
\end{equation*}
\end{lemma}

The proof can be found in Appendix~\ref{app:ridgenotworseproof}.

We will treat each bound separately, and later join the results into a single lemma.
First, we prove that even if we do not have an estimate for $\norm{\expectationD{\vec{x}^{2}}}{\frac{1}{2}}$, with a proper choice of $\eta$, our Two-Phased DDAERR algorithm still performs with probability $1$  over the first phase as well as the AERR algorithm, up to a constant factor.

\begin{lemma}\label{lm:onlineridgenotworse}

Let $\bar{\vec{w}}$ be
the output of Two-Phased DDAERR when run with $\eta=\sqrt{\frac{k}{6dm_{2}}}$. Then with probability $1$ over the first phase, we have for all $m_1$ and for any $\vec{w}^{*}\in\mathbb{R}^{d}$ with $\left\|\vec{w}^{*}\right\|_{2}\leq B$,
\begin{equation*}
\expectationDAphase{2}{L_{\mathcal{D}}\left(\bar{\vec{w}}\right)}-
L_{\mathcal{D}}\left(\vec{w}^{*}\right)\leq\\
4B^2\sqrt{\frac{6d}{km_2}}.
\end{equation*}
\end{lemma}

The proof can be found in appendix \ref{app:onlineridgenotworse}.

Now, if we do have an estimate $H\geq\norm{\expectationD{\vec{x}^{2}}}{\frac{1}{2}}$, we can use it to calculate an appropriate step size and to bound the risk, as shown in the next lemma. 
\begin{lemma}\label{lm:onlineridgepartial}
Assume we have a value $H$ that satisfies $H\geq\norm{\expectationD{\vec{x}^{2}}}{\frac{1}{2}}$.
Let $\bar{\vec{w}}$ be
the output of Two-Phased DDAERR when run with $\eta=\frac{1}{\sqrt{m_2\left(\frac{2}{k}H+ \frac{2}{k}\sqrt{\frac{5}{3}}d\sqrt{H}\sqrt{\epsilon}+1\right)}}$. Then with probability $\geq1-\delta$ over the first phase, we have for all $m_1$ and for any $\vec{w}^{*}\in\mathbb{R}^{d}$ with $\left\|\vec{w}^{*}\right\|_{2}\leq B$,
\begin{equation*}
\expectationDAphase{2}{L_{\mathcal{D}}\left(\bar{\vec{w}}\right)}-
L_{\mathcal{D}}\left(\vec{w}^{*}\right)\leq\\
\frac{4B^2}{\sqrt{m_2}}\sqrt{\frac{2}{k}H+\frac{2}{k}\sqrt{\frac{5}{3}}d\sqrt{H}\sqrt{\epsilon}+1}. \end{equation*}
\end{lemma}

The proof can be found in appendix \ref{app:onlineridgepartialproof}.

This lemma gives a non-trivial expected risk bound only if $\epsilon$ is small enough, but when $m_1$ is small, this is not necessarily the case. Therefore, we would like to unite these two lemmas to ensure that even in the worst case, we won't have a worse bound than the AERR algorithm.
\begin{lemma}\label{lm:onlineridgepartialcombined}
Assume we know a value $H$ that satisfies $H\geq\norm{\expectationD{\vec{x}^{2}}}{\frac{1}{2}}$.
Let $\bar{\vec{w}}$ be the output of Two-Phased DDAERR when run with 
\begin{equation*}
\eta=\max\left(\sqrt{\frac{k}{6dm_{2}}},\sqrt{\frac{k}{m_2\left(2H+ 2\sqrt{\frac{5}{3}}d\sqrt{H}\sqrt{\frac{d\log{\frac{2d}{\delta}}}{\left(k+1\right)m_1}}+k\right)}}\right)\text{.}
\end{equation*}
Then for all $m_1$ and for any $\vec{w}^{*}\in\mathbb{R}^{d}$ with $\norm{\vec{w}^{*}}{2}\leq B$, with probability $1$ over the first phase, we have
\begin{equation*}
\expectationDAphase{2}{L_{\mathcal{D}}\left(\bar{\vec{w}}\right)}-
L_{\mathcal{D}}\left(\vec{w}^{*}\right)\leq
\frac{4B^2}{\sqrt{m_2}}\sqrt{\frac{6d}{k}}. \end{equation*}
Also, with probability $\geq 1-\delta$ over the first phase,
we have\begin{equation*}
\expectationDAphase{2}{L_{\mathcal{D}}\left(\bar{\vec{w}}\right)}-
L_{\mathcal{D}}\left(\vec{w}^{*}\right)\leq\\
\frac{4B^2}{\sqrt{m_2}}\sqrt{\frac{2}{k}H+\frac{2}{k}\sqrt{\frac{5}{3}}d\sqrt{H}\sqrt{\frac{d\log{\frac{2d}{\delta}}}{\left(k+1\right)m_1}}+1}.
\end{equation*}

\end{lemma}

The proof can be found in appendix \ref{app:onlineridgepartialcombined}.

We could always naively bound $\norm{\expectationD{\vec{x}^{2}}}{\frac{1}{2}}$ by $d$, but then, even if $m_1$ tends to infinity, the bound will not be better than the bound of the AERR algorithm. However, if we do have prior knowledge upon the value $\norm{\expectationD{\vec{x}^{2}}}{\frac{1}{2}}$, it is straightforward to use this lemma to prove Theorem~\ref{thm:onlineridgepartialfinal}.

\subsection{Unknown Second Moment Scenario}
In this section, we analyze the case in which we do not have prior knowledge of the second moments of the attributes and on the value of $\norm{\expectationD{\vec{x}^{2}}}{\frac{1}{2}}$.  This scenario may accrue if the learner is faced with a new learning task, and knows nothing about the distribution of the attributes. The problem here, besides not being able to calculate the optimal $q_i$-s, is that we also cannot calculate the optimal step size, $\eta$.

Our solution to this scenario is to use again the Two-Phased DDAERR algorithm (Algorithm~\ref{alg:onlineDDAERR}), and calculate an accurate enough estimation of $\norm{\expectationD{\vec{x}^{2}}}{\frac{1}{2}}$.

\begin{lemma}\label{lm:halfnormestimate}
We can estimate $\norm{\expectationD{\vec{x}^{2}}}{\frac{1}{2}}$ by the estimator $H=\norm{2\vec{A}+\frac{10}{3}\vec{\epsilon}}{\frac{1}{2}}$, which satisfies with probability $\geq1-\delta$, $\norm{\expectationD{\vec{x}^{2}}}{\frac{1}{2}}\leq\norm{2\vec{A}+\frac{10}{3}\vec{\epsilon}}{\frac{1}{2}}\leq8\norm{\expectationD{\vec{x}^{2}}}{\frac{1}{2}}+\frac{34}{3}d^2\vec{\epsilon}$.
\end{lemma}

\begin{proof}
First, using the second inequality in equation \eqref{eq:smallbounds} we have with probability $\geq1-\delta$, that $\norm{\expectationD{\vec{x}^{2}}}{\frac{1}{2}}\leq\norm{2\vec{A}+\frac{10}{3}\vec{\epsilon}}{\frac{1}{2}}$.
Using the first inequality in equation \eqref{eq:smallbounds} and
the identity $\norm{\vec{a}+\vec{b}}{\frac{1}{2}}\leq2\norm{\vec{a}}{\frac{1}{2}}+2\norm{\vec{b}}{\frac{1}{2}}$
we can see that with probability $\geq1-\delta$,
\begin{equation}\let\veqno\leqno
\norm{2\vec{A}+\frac{10}{3}\vec{\epsilon}}{\frac{1}{2}}\leq
\norm{4\expectationD{\vec{x}^{2}}+\frac{14}{6}\vec{\epsilon}+\frac{10}{3}\vec{\epsilon}}{\frac{1}{2}}\leq
8\norm{\expectationD{\vec{x}^{2}}}{\frac{1}{2}}+\frac{34}{3}d^2\vec{\epsilon}.
\qedhere
\end{equation}
\end{proof}
Using this estimate we can prove our main theorem of this section.
\begin{theorem}\label{thm:onlineridge}
Assume the distribution $\mathcal{D}$ is such that $\norm{\vec{x}}{2}\leq1$
and $\left|y\right|\leq B$ with probability 1. 
Let $\bar{\vec{w}}$ be the output of Two-Phased DDAERR when run with 
\begin{equation*}
\eta=\max\left(\sqrt{\frac{k}{6dm_{2}}},\sqrt{\frac{k}{m_2\left(2\norm{2\vec{A}+\frac{10}{3}\vec{\epsilon}}{\frac{1}{2}}+ 2\sqrt{\frac{5}{3}}d\sqrt{\norm{2\vec{A}+\frac{10}{3}\vec{\epsilon}}{\frac{1}{2}}}\sqrt{\frac{d\log{\frac{2d}{\delta}}}{\left(k+1\right)m_1}}+k\right)}}\right)\text{.}
\end{equation*}
Then for all $m_1$ and for any $\vec{w}^{*}\in\mathbb{R}^{d}$ with $\left\|\vec{w}^{*}\right\|_{2}\leq B$, with probability $1$ over the first phase, we have 
\begin{equation*}
\expectationDAphase{2}{L_{\mathcal{D}}\left(\bar{\vec{w}}\right)}-
L_{\mathcal{D}}\left(\vec{w}^{*}\right)\leq
\frac{4B^2}{\sqrt{m_2}}\sqrt{\frac{6d}{k}}. \end{equation*}
Also, with probability $\geq 1-\delta$ over the first phase, we have
\begin{equation*}
\expectationDAphase{2}{L_{\mathcal{D}}\left(\bar{\vec{w}}\right)}-
L_{\mathcal{D}}\left(\vec{w}^{*}\right)\leq
\frac{16B^2}{\sqrt{m_2}}\sqrt{\frac{1}{k}\left(\sqrt{\norm{\expectationD{\vec{x}^{2}}}{\frac{1}{2}}}+d\sqrt{\frac{2d\log{\frac{2d}{\delta}}}{\left(k+1\right)m_1}}\right)^2+1}. \end{equation*}
\end{theorem}

If we examine the bound we can see that with probability $\geq1-\delta$ over the first phase, regardless of the value of $m_1$, the expected risk bound is at most $\BigO{\frac{B^2}{\sqrt{km_2}}\sqrt{d}}$, which is the same bound of the AERR algorithm. This means that the Two-Phased DDAERR algorithm performs with high probability over the first phase as well as the AERR algorithm, up to a constant factor. Second, as $m_1$ increases, the expected risk bound turns to $\BigO{\frac{B^2}{\sqrt{km_2}}\sqrt{\left(\sqrt{\norm{\expectationD{\vec{x}^{2}}}{\frac{1}{2}}}+d\sqrt{\frac{d\log{\frac{2d}{\delta}}}{\left(k+1\right)m_1}}\right)^2+k}}$. Therefore, if $m_1\gg\frac{d^2\log{\frac{2d}{\delta}}}{k+1}$, we achieve an improvement over the AERR algorithm. If $m_1\geq\frac{d^3\log{\frac{2d}{\delta}}}{\left(k+1\right)\norm{\expectationD{\vec{x}^{2}}}{\frac{1}{2}}}$, the bound becomes $\BigO{\frac{B^2}{\sqrt{km_2}}\sqrt{\norm{\expectationD{\vec{x}^{2}}}{\frac{1}{2}}+k}}$, which is the same bound as in the regular DDAERR algorithm with prior knowledge of the second moment of the attributes.

The conclusion is that even if we do not have prior knowledge of
the second moments of the attributes and on $\norm{\expectationD{\vec{x}^{2}}}{\frac{1}{2}}$, we still should prefer our Two-Phased DDAERR algorithm over the AERR algorithm.

It is interesting to compare between the known $\norm{\expectationD{\vec{x}^{2}}}{\frac{1}{2}}$ case and the unknown $\norm{\expectationD{\vec{x}^{2}}}{\frac{1}{2}}$ case. Even though the sampling probabilities are the same, in the unknown $\norm{\expectationD{\vec{x}^{2}}}{\frac{1}{2}}$ case we need $\frac{d}{\norm{\expectationD{\vec{x}^{2}}}{\frac{1}{2}}}$ more samples to reach the regime where our algorithm significantly improves on AERR. The reason for this is that the expected risk bound is highly dependent on the choice of the step size $\eta$, and calculating the optimal value requires knowledge of $\norm{\expectationD{\vec{x}^{2}}}{\frac{1}{2}}$ which is harder to estimate than the moments themselves, as the estimation errors of attributes build up.

\subsubsection{Proof of Theorem~\ref{thm:onlineridge}}
The proof is straightforward using Lemma~\ref{lm:onlineridgepartialcombined}. First, by denoting $H=\norm{2\vec{A}+\frac{10}{3}\vec{\epsilon}}{\frac{1}{2}}$, and using
\begin{equation*}
\eta=\frac{1}{\sqrt{m_2\left(\frac{2}{k}H+ \frac{2}{k}\sqrt{\frac{5}{3}}d\sqrt{H}\sqrt{\epsilon}+1\right)}}=\frac{1}{\sqrt{m_2\left(\frac{2}{k}\norm{2\vec{A}+\frac{10}{3}\vec{\epsilon}}{\frac{1}{2}}+ \frac{2}{k}\sqrt{\frac{5}{3}}d\sqrt{\norm{2\vec{A}+\frac{10}{3}\vec{\epsilon}}{\frac{1}{2}}}\sqrt{\epsilon}+1\right)}}
\end{equation*}
We can see that\begin{align*}
\begin{split}
&\frac{4B^2}{\sqrt{m_2}}\sqrt{\frac{2}{k}H+\frac{2}{k}\sqrt{\frac{5}{3}}d\sqrt{H}\sqrt{\epsilon}+1}\\
&\quad\quad\leq\frac{4B^2}{\sqrt{m_2}}\sqrt{\frac{2}{k}\norm{2\vec{A}+\frac{10}{3}\vec{\epsilon}}{\frac{1}{2}}+\frac{2}{k}\sqrt{\frac{5}{3}}d\sqrt{\norm{2\vec{A}+\frac{10}{3}\vec{\epsilon}}{\frac{1}{2}}}\sqrt{\epsilon}+1}\\
&\quad\quad\leq\frac{4B^2}{\sqrt{m_2}}\sqrt{\frac{2}{k}\left(8\norm{\expectationD{\vec{x}^{2}}}{\frac{1}{2}}+\frac{34}{3}d^2\vec{\epsilon}\right)+\frac{2}{k}\sqrt{\frac{5}{3}}d\sqrt{\left(8\norm{\expectationD{\vec{x}^{2}}}{\frac{1}{2}}+\frac{34}{3}d^2\vec{\epsilon}\right)}\sqrt{\epsilon}+1}\\
&\quad\quad\leq\frac{4B^2}{\sqrt{m_2}}\sqrt{\frac{16}{k}\norm{\expectationD{\vec{x}^{2}}}{\frac{1}{2}}+\frac{68}{3k}d^2\vec{\epsilon}+\frac{2}{k}\sqrt{\frac{40}{3}}d\sqrt{\norm{\expectationD{\vec{x}^{2}}}{\frac{1}{2}}}\sqrt{\epsilon}+\frac{2}{k}\sqrt{\frac{170}{9}}d^2\epsilon+1}\\
&\quad\quad\leq\frac{4B^2}{\sqrt{m_2}}\sqrt{\frac{1}{k}\left(16\norm{\expectationD{\vec{x}^{2}}}{\frac{1}{2}}+2\sqrt{\frac{40}{3}}d\sqrt{\norm{\expectationD{\vec{x}^{2}}}{\frac{1}{2}}}\sqrt{\epsilon}+32d^2\epsilon\right)+1}\\
&\quad\quad\leq\frac{16B^2}{\sqrt{m_2}}\sqrt{\frac{1}{k}\left(\sqrt{\norm{\expectationD{\vec{x}^{2}}}{\frac{1}{2}}}+d\sqrt{2\epsilon}\right)^2+1}. \end{split}
\end{align*}
Using Lemma~\ref{lm:onlineridgepartialcombined} and plugging in $\epsilon=\frac{d\log{\frac{2d}{\delta}}}{\left(k+1\right)m_1}$  finishes the proof.

\section{Attribute Efficient Lasso Regression}\label{sec:lasso}
In this section we present our algorithms for lasso regression, where the loss is again the squared loss, $\ell\left(\vec{w};\vec{x}_t,y_t\right)=\frac{1}{2}\left(\left<\vec{w},\vec{x}_t\right>-y_t\right)^2$, and the 1-norm bound is $\norm{\vec{w}}{1}\leq B$. The generic approach to the lasso attribute efficient scenario, which we call the General Attribute Efficient Lasso Regression (GAELR) algorithm and is presented in Algorithm \ref{alg:gaelr}, was first developed in \cite{hazan} and  is based on a stochastic variant of the Exponentiated Gradient (EG) algorithm with gradient estimates, developed in \cite{warmuth}. 

The EG algorithm goes over the training set, and for each example builds an unbiased estimator of the gradient and clips it (where the $clip$ operation is defined as $clip(x,c)=\max\left\{\min\left\{x,c\right\},-c\right\}$) to make the updates more robust. 
Afterwards, the algorithm updates $\vec{w}_t$ by performing multiplicative updates of size $\eta$. The result is projected over the $L_1$ ball of size $B$, yielding $\vec{w}_{t+1}$. At the end, the algorithm outputs the average of all $\vec{w}_t$. The algorithm converges to the global minimum, as the minimization problem is convex in $\vec{w}$.

The GAELR algorithm build the unbiased gradient estimates similarly to the
GAERR algorithm, with a slight modification: When estimating the inner product,
instead of sampling one sample with probability $p_{j_{t}}=\frac{w_{t,j_t}^2}{\norm{\vec{w}_t}{2}^2}$,
it samples it with probability $p_{j_{t}}=\frac{\left|\vec{w}_t\left[j_t\right]\right|}{\norm{\vec{w}_t}{1}}$,
as the lasso scenario has a bound on the 1-norm of the predictor.

\begin{algorithm}[h!]
  \caption{GAELR \newline 
  Parameters: $B,\eta>0$ and $q_i$ for $i\in\left[d\right]$}
  \label{alg:gaelr}
  \begin{algorithmic}[1]
  \REQUIRE training set $S=\left\{\left(\vec{x}_{t},y_{t}\right)\right\}_{t\in\left[m\right]}$ and $k>0$
  \ENSURE regressor $\bar{\vec{w}}$ with $\norm{\bar{\vec{w}}}{1}\leq B$
  \STATE Initialize $\vec{z}_1^+\leftarrow\vec{1}_d,\vec{z}_1^-\leftarrow\vec{1}_d$
 
  \FOR{$t=1$ to $m$}
   \STATE $\vec{w}_t\leftarrow\left(\vec{z}_t^+-\vec{z}_t^-\right)B/\left(\norm{\vec{z}_t^+}{1}+\norm{\vec{z}_t^-}{1}\right)$
    \FOR{$r=1$ to $k$}
    \STATE Pick $i_{t,r}\in \left[d\right]$ with probability $q_{i_{t,r}}$ and observe $\vec{x}_t\left[i_{t,r}\right]$
    \STATE $\widetilde{\vec{x}}_{t,r}\leftarrow\frac{1}{q_{i_{t,r}}}\vec{x}_{t}\left[i_{t,r}\right] \cdot \vec{e}_{i_{t,r}}$

    \ENDFOR
  \STATE $\widetilde{\vec{x}}_{t} \leftarrow \frac{1}{k}\sum_{r=1}^{k}\widetilde{\vec{x}}_{t,r}$
  
  \STATE Choose $j_{t}\in\left[d\right]$ with probability $p_{j_{t}}=\frac{\left|\vec{w}_t\left[j_t\right]\right|}{\norm{\vec{w}_t}{1}}$
   and observe $\vec{x}_{t}\left[j_{t}\right]$
  \STATE $\widetilde{\phi}_{t} \leftarrow \frac{w_{t,j}}{p_{j}}\vec{x}_{t}\left[j_{t}\right]-y_t$
  \STATE $\widetilde{\vec{g}}_{t} \leftarrow \widetilde{\phi}_{t} \cdot \widetilde{\vec{x}}_{t}$
  \FOR{$i=1$ to $d$}
    \STATE $\vec{\bar{g}}_t\left[i\right]=clip\left(\vec{\widetilde{g}}_t\left[i\right],1/\eta\right)$
    \STATE $\vec{z}_{t+1}^+\left[i\right]\leftarrow\vec{z}_t^+\left[i\right]\cdot\exp\left(-\eta \vec{\bar{g}}_t\left[i\right]\right)$
    \STATE $\vec{z}_{t+1}^-\left[i\right]\leftarrow\vec{z}_t^-\left[i\right]\cdot\exp\left(+\eta \vec{\bar{g}}_t\left[i\right]\right)$
  \ENDFOR
  \ENDFOR
  \STATE $\bar{\vec{w}} \leftarrow \frac{1}{m}\sum_{t=1}^{m}\vec{w}_{t}$
  \end{algorithmic}
\end{algorithm}

The expected risk bound of the GAELR algorithm is presented in the next theorem which is a slightly more general version of Theorem~3.4 in \cite{hazan}. 
\begin{theorem}\label{thm:maingaelr}
Assume the distribution $\mathcal{D}$ is such that $\norm{\vec{x}}{\infty}\leq1$ and $\left|y\right|<B$ with probability 1. Let $\bar{\vec{w}}$ be the output of GAELR, when run with step size $\eta\leq\frac{1}{2G}$ where $\max_t\norm{\expectationDA{\widetilde{\vec{g}_t}^2}}{\infty}\leq
G^2$. Then for any $\vec{w}^{*}\in\mathbb{R}^{d}$ with $\norm{\vec{w}^{*}}{1}\leq
B$,
\begin{equation*}
\expectationDA{L_{\mathcal{D}}\left(\bar{\vec{w}}\right)}\leq
L_{\mathcal{D}}\left(\vec{w^{*}}\right)+B\left(\frac{\log2d}{\eta m} + 5\eta G^{2}\right).
\end{equation*}
\end{theorem}
The general idea of the proof is that $\widetilde{\vec{g}}_{t}$ is an unbiased estimator of the gradient, therefore we can use the standard analysis of the EG algorithm. The full proof can be found in appendix \ref{app:maingaelrproof}.

The AELR algorithm is one variant of the GAELR algorithm. It was presented in \cite{hazan} and uses uniform sampling to estimate $\vec{x}_t$. In our GAELR notation it uses
\begin{equation*}
q_i=\frac{1}{d}\:\:\:\forall i\in \left[d\right].
\end{equation*}
The authors prove (Lemma~3.8 in \cite{hazan}) that for the AELR algorithm, $G^2\leq8B^2d/k$, which together with Theorem~\ref{thm:maingaelr} and using
$\eta=\frac{2B}{G\sqrt{m}}$ yields an expected risk bound of $4B^2\sqrt{\frac{10d\log2d }{km}}$. 

Similarly to the ridge scenario, by analyzing the bound, we show that we can improve the bound in a data-dependent manner: Theorem~\ref{thm:maingaelr} tells us that the expected risk bound is proportional to $G$, therefore we wish to develop a sampling method that minimizes the infinity norm of the gradient estimator. 

The gradient estimate consist of estimating the inner product and estimating $\vec{x}_t$. The next lemma will assist in bounding the infinity norm of $\widetilde{\vec{x}}_{t}^2$.

\begin{lemma}\label{lm:infinitynormbound}
For every distribution $\left(q_1,..,q_d\right)$ where $q_i\geq0$ and $i\in\left[d\right]$, we have 
$\norm{\expectationDA{\widetilde{\vec{x}}_{t}^2}}{\infty}=
\max_i\frac{1}{k}\expectationDA{\widetilde{\vec{x}}_{t,r}^2\left[i\right]}+\frac{k-1}{k}\expectationD{\norm{\vec{x}}{\infty}}^2$.
\end{lemma}
The proof can be found in appendix \ref{app:infinitynormboundproof}.

Since
\begin{equation}\label{eq:xtrinfinitynorm}
\expectationDA{\widetilde{\vec{x}}_{t,r}^2\left[i\right]}=
\frac{1}{q_i}\expectationD{x_i^2},
\end{equation}
in order to minimize the infinity norm of the estimator, we need to solve the following optimization problem:

\begin{equation*}
\begin{aligned}
& \underset{q_i}{\text{minimize}}
& & \max_i\frac{1}{kq_i}\expectationD{x_i^2}+\frac{k-1}{k}\expectationD{\norm{\vec{x}}{\infty}}^2 \\
& \text{subject to}
& & \sum_{i=1}^dq_i=1,\:\forall i\:q_i\geq0.
\end{aligned}
\end{equation*}
This problem  is equivalent to 
\begin{equation}\label{eq:lassooptimization}
\begin{aligned}
& \underset{q_i}{\text{minimize}}
& & \max_i\frac{1}{q_i}\expectationD{x_i^2} \\
& \text{subject to}
& & \sum_{i=1}^dq_i=1,\:\forall i\:q_i\geq0.
\end{aligned}
\end{equation}

The next lemma gives the optimal value of the $q_i$-s.
\begin{lemma}\label{lm:lassooptimalqi}
The solution to the optimization problem defined in \eqref{eq:lassooptimization}
is $q_i=\frac{\expectationD{x^2_i}}{\sum_{j=1}^d\expectationD{x^2_i}}$.
\end{lemma}
The proof can be found in appendix \ref{app:lassooptimalqiproof}.

The next lemma will assist in bounding the square of the estimator of the inner product (minus the label).

\begin{lemma}\label{lm:twonormboundphilasso}
Using our sampling method we have $\expectationDA{\widetilde{\phi_{t}}^2}\leq4B^2$.
\end{lemma}

The proof can be found in appendix \ref{app:twonormboundphilassoproof}.

As in the ridge scenario, we could have tried to optimized the sampling probabilities
of the inner product estimation. However, since $\expectationDA{\widetilde{\phi_{t}}^2}$ is calculated using the same method as in the ridge scenario, the optimal sampling probabilities remain $p_i=\frac{\sqrt{\vec{w}_{t,i}^2\expectationD{x^2_i}}}{\sum_{j=1}^d\sqrt{\vec{w}_{t,j}^2\expectationD{x^2_j}}}$, but we will still ignore this improvement in our analysis.
 
Altogether, we can formulate a lemma that will bound the gradient estimate.
\begin{lemma}\label{lm:gaelrmain}
The GAELR algorithm generates gradient estimates that for all $t$, $\norm{\expectationDA{\widetilde{\vec{g}_t}^2}}{\infty} \leq 4B^{2}\left(\frac{1}{k}\norm{\expectationDA{\widetilde{\vec{x}}_{t,r}^2}}{\infty}+1\right)$.
\end{lemma}

\begin{proof}
This lemma follows directly from Lemmas \ref{lm:infinitynormbound} and \ref{lm:twonormboundphilasso}, using the independence of $\widetilde{\vec{x}}_{t}$ and $\widetilde{\phi_{t}}$ given $\vec{x}_t$ and $\norm{\vec{x}}{\infty}\leq1$.
\end{proof}

\subsection{Known Second Moment Scenario}
If we assume we have prior knowledge of the second moment of each attribute, namely $\expectationD{x_i^2}$ for all $i\in\left[d\right]$, we can use Lemma~\ref{lm:lassooptimalqi} to calculate the optimal values of the $q_i$-s. This is the idea behind our DDAELR (Data-Dependent Attribute Efficient Lasso Regression) algorithm.

The expected risk bound of the algorithm is formulated in the next theorem.
\begin{theorem}\label{thm:mainlasso}
Assume the distribution $\mathcal{D}$ is such that $\norm{\vec{x}}{\infty}\leq1$
and $\left|y\right|\leq B$ with probability 1 and $\expectationD{x_i^2}$ are known for $i\in\left[d\right]$.
Let $\bar{\vec{w}}$ be
the output of DDAELR, when run with $\eta=\frac{1}{2B}\sqrt{\frac{\log2d}{5m\left(\frac{1}{k}\norm{\expectationD{\vec{x}^{2}}}{1}+1\right)}}$.
If $m\geq\log2d$ then for any $\vec{w}^{*}\in\mathbb{R}^{d}$ with $\left\|\vec{w}^{*}\right\|_{1}\leq
B$,
\begin{equation*}
\expectationDA{L_{\mathcal{D}}\left(\bar{\vec{w}}\right)}\leq
L_{\mathcal{D}}\left(\vec{w^{*}}\right)+4B^2\sqrt{\frac{5\log2d\left(\frac{1}{k}\norm{\expectationD{\vec{x}^{2}}}{1}+1\right)}{m}}.
\end{equation*}
\end{theorem}

\begin{proof}[Proof of Theorem~\ref{thm:mainlasso}]
If $m\geq\log2d$, we have $\eta\leq\frac{1}{2G}$ and the theorem follows directly from Theorem~\ref{thm:maingaelr}, Lemma~\ref{lm:gaelrmain}, equation~\eqref{eq:xtrinfinitynorm} and the calculated $q_i$-s in Lemma~\ref{lm:lassooptimalqi}. \end{proof}

Recalling that with probability 1 we have $\norm{\vec{x}}{\infty}\leq1$, it is easy to see that $\norm{\expectationD{\vec{x}^{2}}}{1}\leq d$, therefore the DDAELR algorithm always performs at least as well as the AELR algorithm\footnote{If $\norm{\expectationD{\vec{x}^{2}}}{1}= d$ it is easy to see that $\expectationD{x_i}=1$ for all $i\in\left[d\right]$. In this case, all the $q_i$-s are equal to $\frac{1}{d}$ and the DDAELR and AELR algorithms coincide.}. However,  $\norm{\expectationD{\vec{x}^{2}}}{1}$ may also be much smaller than $d$, in cases where the second moments varies between attributes or the vector is sparse. In these cases, we may gain a significant improvement. For example, if we consider a harmonic attribute decay such as  $\expectationD{x_i^{2}}=\frac{1}{i}$, we have $\norm{\expectationD{\vec{x}^{2}}}{1}=\BigO{\log d}$ which is significantly smaller than $d$.

\subsection{Unknown Second Moment Scenario}
The solution presented in the previous section requires exact knowledge of $\expectationD{x_i^2}$ for all $i$, which may not be available when the learner is faced with a new learning task. Thus, we turn to consider the case where the moments are initially unknown. 

We take a similar approach to the Two-Phased DDAERR algorithm: in the first phase, we estimate the second moments by uniform sampling, exactly as in the Two-Phased DDAERR algorithm. In the second phase, we run the DAELR with modified $q_i$s which use an upper confidence interval instead of the second moments themselves.
This approach is the basis for our Two-Phased DDAELR algorithm (Algorithm~\ref{alg:onlineDDAELR}). \begin{algorithm}[h!]
  \caption{Two-Phased DDAELR \newline 
  Parameters: $m_1,m_2,\delta,B,\eta>0$}
  \label{alg:onlineDDAELR}
  \begin{algorithmic}[1]
  \REQUIRE training set $S=\left\{\left(\vec{x}_{t},y_{t}\right)\right\}_{t\in\left[m_{1}+m_{2}\right]}$ and $k>0$
  \ENSURE regressor $\bar{\vec{w}}$ with $\left\|\bar{\vec{w}}\right\|\leq B$
  \STATE Initialize $\vec{w_{1}\neq 0}$, $\left\|\vec{w_{1}}\right\|_{2}\leq B$ arbitrarily
  \STATE Initialize $\vec{A}$, $counts$ and $square\_sums$ - arrays of size $d$ with zeros
  \FOR{$t=1$ to $m_{1}$}
    \FOR{$r=1$ to $k+1$}
      \STATE Pick $i_{t,r}\in\left[d\right]$ uniformly at random
      \STATE $counts\left[i_{t,r}\right] \leftarrow counts\left[i_{t,r}\right]+1$
      \STATE $square\_sums\left[i_{t,r}\right] \leftarrow square\_sums\left[i_{t,r}\right]+\vec{x}_{t}\left[i_{t,r}\right]^{2}$
      \STATE $\vec{A}\left[i_{t,r}\right] \leftarrow \frac{square\_sums\left[i_{t,r}\right]}{counts\left[i_{t,r}\right]}$
  \ENDFOR
  \ENDFOR
  \STATE $\epsilon \leftarrow \min\left(\frac{d\log{\frac{2d}{\delta}}}{\left(k+1\right)m_1},1\right)$\label{alg:DDAELR:epsilon} 
  \STATE Run GAELR with $q_i=\frac{\vec{A}\left[i\right]+\frac{13}{6}\epsilon}{\sum_{j=1}^{d}\left(\vec{A}\left[{j}\right]+\frac{13}{6}\epsilon\right)}$ on the following $m_2$ examples and return its output
  \end{algorithmic}
\end{algorithm}

As in the Two-Phased DDAERR algorithm, during the first phase one can actually run the AELR algorithm
in order to obtain a better starting point for the second phase, but we will
ignore this improvement in our analysis.

The expected risk bound of the algorithm is formulated in the following theorem.

\begin{theorem}\label{thm:onlinelassopartialfinal}
Assume the distribution $\mathcal{D}$ is such that $\norm{\vec{x}}{\infty}\leq1$
and $\left|y\right|\leq B$ with probability 1.
Let $\bar{\vec{w}}$ be
the output of DDAELR, when run with $\eta=\sqrt{\frac{k\log2d}{20B^2m_{2}\left(8\norm{\vec{A}}{1}+20d\min\left(\frac{d\log{\frac{2d}{\delta}}}{\left(k+1\right)m_1},1\right)+k\right)}}$.
If~$m_2\geq\log2d$ then for any $m_1$ and for any $\vec{w}^{*}\in\mathbb{R}^{d}$ with $\left\|\vec{w}^{*}\right\|_{1}\leq B$, with probability 1 over the first phase we have,
\begin{equation*}
\expectationDAphase{2}{L_{\mathcal{D}}\left(\bar{\vec{w}}\right)}-
L_{\mathcal{D}}\left(\vec{w}^{*}\right)\leq\\
61B^2\sqrt{\frac{d\log2d}{km_2}}.
\end{equation*}
Also, with probability $1-\delta$ over the first phase we have
\begin{equation*}
\expectationDAphase{2}{L_{\mathcal{D}}\left(\bar{\vec{w}}\right)}-
L_{\mathcal{D}}\left(\vec{w}^{*}\right)\leq\\
4B^2\sqrt{\frac{5\left(16\norm{\expectationD{\vec{x}^{2}}}{1}+\frac{}{}\frac{88d}{3}\min\left(\frac{d\log{\frac{2d}{\delta}}}{\left(k+1\right)m_1},1\right)+k\right)\log2d}{km_2}}. \end{equation*}
\end{theorem}

With probability $1$ over the first phase, regardless of the value of $m_1,$ the expected risk bound is at most $\BigO{\frac{B^2}{\sqrt{km_2}}\sqrt{d\log d}}$, which is the same bound of the AELR algorithm. This means that the Two-Phased DDAELR algorithm performs with probability $1$ over the first phase as well as the AELR algorithm, up to a constant factor. Second, as $m_1$ increases, the expected risk bound becomes $\BigO{\frac{B^2}{\sqrt{km_2}}\sqrt{\left(\norm{\expectationD{\vec{x}^{2}}}{1}+\frac{d^2\log{\frac{2d}{\delta}}}{\left(k+1\right)m_1}+k\right)\log d}}$. Therefore, if $m_1\gg\frac{d\log{\frac{2d}{\delta}}}{k+1}$, we achieve an improvement over the AELR algorithm. If $m_1\geq\frac{d^2\log{\frac{2d}{\delta}}}{\left(k+1\right)\norm{\expectationD{\vec{x}^{2}}}{1}}$, the expected risk bound turns to $\BigO{\frac{B^2}{\sqrt{km_2}}\sqrt{\norm{\expectationD{\vec{x}^{2}}}{1}+k}}$, which is the same bound as in the regular DDAELR algorithm with prior knowledge of the second moment of the attributes.

The conclusion is that even if we do not have prior knowledge of
the second moments of the attributes, we still should prefer our Two-Phased DDAELR algorithm over the AELR algorithm.

It is also interesting to compare these improvement regimes to those of the Two-Phased DDAERR algorithm. If we analogize $\norm{\expectationD{\vec{x}^{2}}}{\frac{1}{2}}$ to $\norm{\expectationD{\vec{x}^{2}}}{1}$, the regimes for the lasso scenario are better by a factor $d$ than the corresponding regimes for the ridge scenario. The reason for this is that for any $i$, $\expectationD{x_i^2}$ is much easier to estimate by sampling than $\sqrt{\expectationD{x_i^2}}$, because the square root is not a Lipschitz function.

\subsubsection{Proof of Theorem~\ref{thm:onlinelassopartialfinal}}

The main goal of the proof is to bound the expected squared infinity-norm of the gradient estimator from above. By using Lemma~\ref{lm:gaelrmain}, all that remains is to upper bound $\norm{\expectationDA{\widetilde{\vec{x}}_{t,r}^2}}{\infty}$ as we do in the next lemma.
\begin{lemma}\label{lm:lassonotworse}
For all $t>m_1$, the bound
\begin{equation*}
\norm{\expectationDAphase{2}{\widetilde{\vec{x}}_{t,r}^2}}{\infty} \leq 4\norm{\expectationD{\vec{x}^2}}{1}+\frac{20}{3}d\epsilon
\end{equation*}
holds with probability $1$ if $\epsilon=1$ and with probability $\geq 1-\delta$, if $\epsilon\leq1$.
\end{lemma}

The proof can be found in Appendix~\ref{app:lassonotworseproof}.

In the lasso scenario it is sufficient to use one bound (compare to Lemma~\ref{lm:ridgenotworse} in the ridge scenario) as we are able to join the two regimes of $\epsilon$ by ensuring $\epsilon\leq1$ (Algorithm \ref{alg:onlineDDAELR}, line \ref{alg:DDAELR:epsilon}).
Using this bound, the proof of the theorem is straightforward. First, using Theorem~\ref{thm:maingaelr} on the second phase of the algorithm, we have
\begin{equation}\label{eq:onlineDDAELRbound}
\expectationDAphase{2}{L_{\mathcal{D}}\left(\bar{\vec{w}}\right)}-
L_{\mathcal{D}}\left(\vec{w^{*}}\right)\leq B\left(\frac{\log2d}{\eta m_{2}} + 5\eta G^{2}\right).
\end{equation}
Now we use Lemma~\ref{lm:lassonotworse}, plug it into Lemma~\ref{lm:gaelrmain} and have $G^2 \leq 4B^{2}\left(\frac{4}{k}\norm{\expectationD{\vec{x}^2}}{1}+\frac{20}{3k}d\epsilon+1\right)$ with probability $1$ if $\epsilon=1$ and with probability $\geq1-\delta$, if $\epsilon\leq1$.
We continue by denoting $\widehat{G^2}=4B^{2}\left(\frac{4}{k}\norm{2\vec{A}+\frac{10}{3}\epsilon}{1}+\frac{20}{3k}d\epsilon+1\right)$
and by using equation~\eqref{eq:smallbounds} we obtain $G^2\leq\widehat{G^2}$. Plugging $\eta=\sqrt{\frac{\log2d}{\widehat{G^2}5m_2}}=\sqrt{\frac{k\log2d}{20B^2m_{2}\left(8\norm{\vec{A}}{1}+20d\epsilon+k\right)}}$ into equation \eqref{eq:onlineDDAELRbound},
we have
\begin{align*}
\expectationDAphase{2}{L_{\mathcal{D}}\left(\bar{\vec{w}}\right)}-
L_{\mathcal{D}}\left(\vec{w^{*}}\right)
&\leq B\left(\frac{\log2d}{m_2\eta} + 5\eta G^{2}\right)\\
&\leq B\left(\frac{\log2d}{m_2\eta} + 5\eta \widehat{G^{2}}\right)\\
&\leq 2B\sqrt{\frac{5\widehat{G^{2}}\log2d}{m_2}}\\
&\leq 4B^2\sqrt{\frac{5\left(4\norm{2\vec{A}+\frac{10}{3}\epsilon}{1}+\frac{20}{3}d\epsilon+k\right)\log2d}{km_2}}.
\end{align*}
Using
\begin{equation}
\norm{2\vec{A}+\frac{10}{3}\epsilon}{1}\leq
\norm{4\expectationD{\vec{x}^{2}}+\frac{14}{6}\epsilon+\frac{10}{3}\epsilon}{1}\leq
4\norm{\expectationD{\vec{x}^{2}}}{1}+\frac{17}{3}d\epsilon,
\end{equation}
we have
\begin{align*}
\expectationDAphase{2}{L_{\mathcal{D}}\left(\bar{\vec{w}}\right)}-
L_{\mathcal{D}}\left(\vec{w^{*}}\right)
&\leq4B^2\sqrt{\frac{5\left(16\norm{\expectationD{\vec{x}^{2}}}{1}+\frac{68}{3}d\epsilon+\frac{20}{3}d\epsilon+k\right)\log2d}{km_2}}\\
&\leq4B^2\sqrt{\frac{5\left(16\norm{\expectationD{\vec{x}^{2}}}{1}+\frac{88}{3}d\epsilon+k\right)\log2d}{km_2}}.
\end{align*}

If $\epsilon=1$, we have 
\begin{equation*}
4B^2\sqrt{\frac{5\left(16\norm{\expectationD{\vec{x}^{2}}}{1}+\frac{88}{3}d\epsilon+k\right)\log2d}{km_2}}\leq61B^2\sqrt{\frac{d\log2d}{km_2}}
\end{equation*}
with probability 1.
Otherwise plugging in $\epsilon=\min\left(\frac{d\log{\frac{2d}{\delta}}}{\left(k+1\right)m_1},1\right)$ finishes the proof.

\section{Experiments}\label{sec:experiments}
In this section we describe some experiments designed to test our algorithms and substantiate our analytical claims. We conducted 3 sets
of experiments: on an artificial data set that allows us to easily control the properties of the data such as $\norm{\expectationD{\vec{x}^{2}}}{\frac{1}{2}}$ and $\norm{\expectationD{\vec{x}^{2}}}{1}$ and to show the dependance of the algorithms on them; on a subset of the popular MNIST \cite{mnist} data set, containing only the "3" and "5" digits, similar to \cite{ohad,hazan}; and on the Covertype \cite{covertype} data set. MNIST and Covertype
 were designed for a binary classification task, which was addressed  by regressing on the -1 and +1 labels.

For the ridge regression scenario, each test consists of 5 algorithms:
\begin{enumerate}
\item 
Our DDAERR algorithm that has prior knowledge of the second moment of the attributes.
\item
Our Two-Phased DDAERR algorithm that does not have prior knowledge of the second moments of the attributes, and tries to estimate them. \item 
The AERR algorithm that does not require any prior knowledge.
\item
Online ridge regression that performs online gradient descent and has access to all the attributes.
\item
Offline ridge regression that minimizes
the empirical risk, which also has access to all attributes, and moreover, utilized the data better than the online algorithm, as it uses each training example more than once.
\end{enumerate}

For the lasso scenario we used the corresponding algorithms. In all cases our algorithms used the improved inner product estimation as well as the improved data point estimation, as discussed in page~\pageref{par:innerproductimprovement}.

For a fair comparison between the attribute efficient algorithms and the
full-information algorithms, we use the X-axis in our figures to represent the number of
attributes each algorithm sees, and not the number of examples that is usually used in these kinds of comparisons. The reason for this is that we would like to compare the algorithms by the total budget they use.

To quantify the theoretical improvement of the DDAERR algorithm, we need to compare
$\norm{\expectationD{\vec{x}^2}}{\frac{1}{2}}$ to $d$, as this is the potential improvement according to our analysis. To avoid scaling issues,
we also normalize by $\expectationD{\norm{\vec{x}^2}{1}}$, and define our
'Improvement Ratio' by
\begin{equation}
\rho_\text{ridge}=\frac{\norm{\expectationD{\vec{x}^2}}{\frac{1}{2}}}{d\expectationD{\norm{\vec{x}^2}{1}}}
\end{equation}
We prefer this definition upon a simpler definition using the exact bound ratio of the different
algorithms because we want to emphasize that this quantity is
a property of the data set itself, and is not algorithm nor analysis dependent.

Similarly, for the lasso scenario, we define
\begin{equation}
\rho_\text{lasso}=\frac{\norm{\expectationD{\vec{x}^2}}{1}}{d\norm{\expectationD{\vec{x}^2}}{\infty}}
\end{equation}

For each data set and algorithm we have used 10-fold cross validation, similar to \cite{ohad,hazan}, to optimize the parameters for each phase, and run the learning process 100 times on increasingly long prefixes of the training set to obtain a sense of the variability of the results. We measured the performance of each algorithm by the average loss over the testing set, divided by the loss of the zero predictor, and defined the error bars as one standard deviation. For the two-phased algorithms, we set $m_1=\frac{m}{10}$, $m_2=\frac{9m}{10}$,
and run the AERR/AELR algorithm during the first phase, using its result
as a starting point for the second phase. Unlike the theoretical analysis, we set $\epsilon$ to be $0$, as the theoretical upper confidence bound is conservative and we found that this  improves the empirical results (though increases their variability).
We have also split the attribute budget evenly between the data point estimation
and the inner product estimation, as it improved the empirical results as
well.

\subsection{Simulated Data}\label{sec:artificial}
We begin by studying a synthetic linear regression data set that easily allows us to control the improvement ratio in both scenarios and to demonstrate the dependence of the algorithms on them. For each experiment,  we first defined a vector $\vec{u}\in\mathbb{R}^d$ for $d=500$ by an exponent decaying factor: $u_i=i^{\alpha}$ for some $\alpha\leq0$ and then projected the vector on the $L_2$ ball of radius $1$ for the ridge scenario (and on the $L_\infty$ ball of radius $1$ for the lasso scenario), to produce the expected values of each attribute, namely the vector $\expectation{\vec{x}}$. To generate one training example, we generated independent binary variables with the corresponding expectations, and joined them into one $d$-dimensional vector. To generate the entire training set, we repeated the example generation process independently $m$ times,
where in each experiment we used a different $m$ to emphasis the interesting
regime. For all these experiments, we used $k+1=5$.

For the ridge scenario, the target values were generated using a scalar product with a random weight vector from $\{-1,1\}^d$, $\vec{w}_\text{ridge}^*$, which itself was generated i.i.d. with $P\left(w^*_{\text{ridge}, i}=1\right)=P\left(w^*_{\text{ridge}, i}=-1\right)=0.5$. For the lasso scenario, the target values were generated using a scalar product with a random sparse weight vector from $\{-1,0,1\}^d$, $\vec{w}_\text{lasso}^*$, which was generated i.i.d. with $P\left(w_{\text{lasso}, i}^*=1\right)=P\left(w_{\text{lasso}, i}^*=-1\right)=0.15$ and $P\left(w_{\text{lasso}, i}^*=0\right)=0.7$.
\begin{figure}
        \centering
        \begin{subfigure}[t]{0.48\textwidth}
                \includegraphics[width=\textwidth]{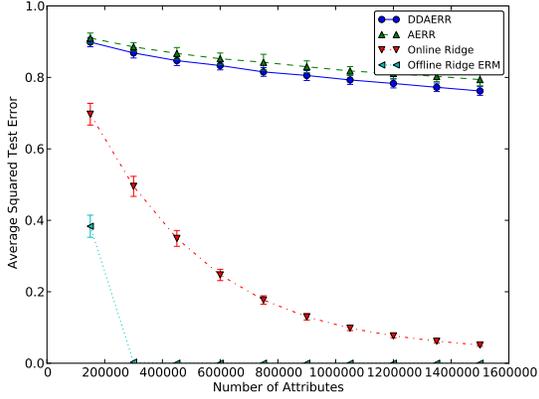}
                \caption{The second moments were chosen to be equal, which results in an improvement factor of $\rho_\text{ridge}=1$.}
                \label{fig:simulated_ridge_zero}
        \end{subfigure}
        \quad
        \begin{subfigure}[t]{0.48\textwidth}
                \includegraphics[width=\textwidth]{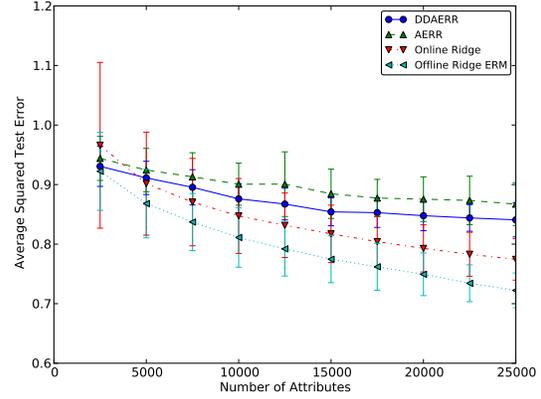}
                \caption{The second moments were chosen by a power law with an exponent of $\alpha=-0.5$, which results in an improvement factor of $\rho_\text{ridge}=0.91$.}
                \label{fig:simulated_ridge_half}
        \end{subfigure}

        \begin{subfigure}[t]{0.48\textwidth}
                \includegraphics[width=\textwidth]{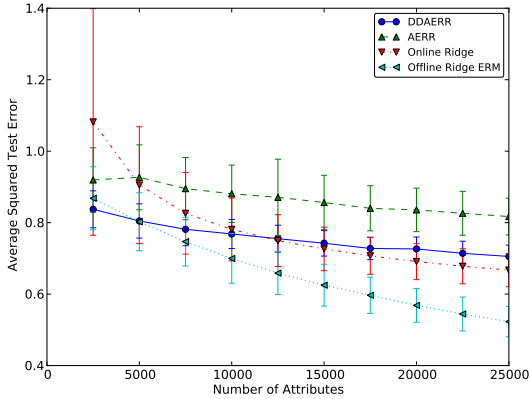}
                \caption{The second moments were chosen by a power law with an exponent of $\alpha=-1$, which results in an improvement factor of $\rho_\text{ridge}=0.55$.}
                \label{fig:simulated_ridge_one}
        \end{subfigure}
        \quad
        \begin{subfigure}[t]{0.48\textwidth}
                \includegraphics[width=\textwidth]{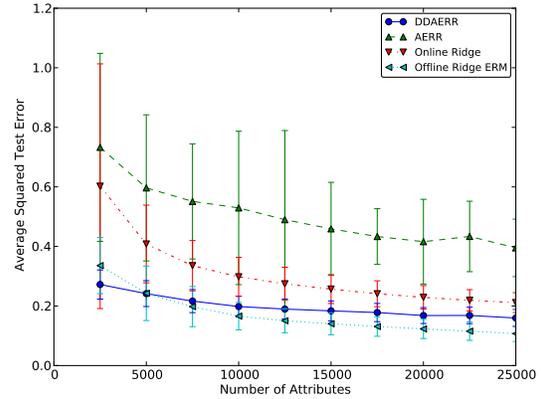}
                \caption{The second moments were chosen by a power law with an exponent of $\alpha=-2$, which results in an improvement factor of $\rho_\text{ridge}=0.05$.}
                \label{fig:simulated_ridge_two}
        \end{subfigure}
        \caption{Test error for the algorithms with $k+1=5$ in the ridge scenario over simulated data with $d=500$.}
        \label{fig:simulated_ridge}
\end{figure}

The results for the ridge scenario appear in figure~\ref{fig:simulated_ridge}: In the first experiment, all the attributes have the same distribution, therefore we have $\rho_\text{ridge}=1$, and the DDAERR and AERR algorithms are equivalent\footnote{The small difference between the algorithms is caused by the difference between the methods each algorithm uses when calculating the optimal step size, $\eta$.}. As $\rho_\text{ridge}$ decreases, the algorithms drift apart, and we see a significant improvement in our methods as predicted by our theory.

\begin{figure}
        \centering
        \begin{subfigure}[t]{0.48\textwidth}
                \includegraphics[width=\textwidth]{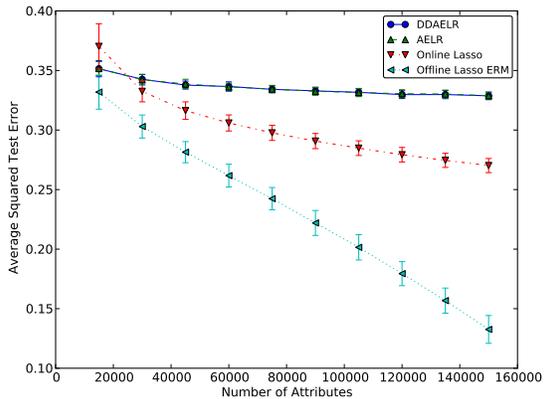}
                \caption{The second moments were chosen to be equal, which results in an improvement factor of $\rho_\text{lasso}=1$.}
                \label{fig:simulated_lasso_zero}
        \end{subfigure}
        \begin{subfigure}[t]{0.48\textwidth}
                \includegraphics[width=\textwidth]{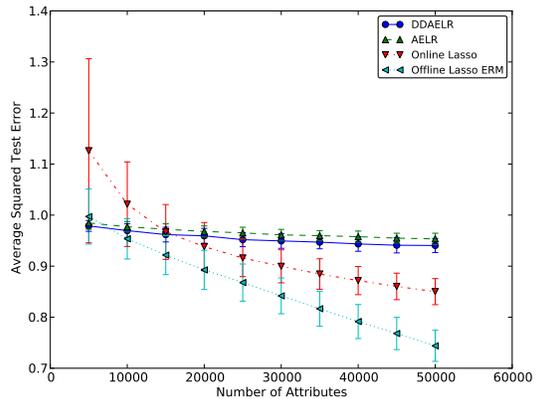}
                \caption{The second moments were chosen by a power law with an exponent of $\alpha=-0.5$, which results in an improvement factor of $\rho_\text{lasso}=0.086$.}
                \label{fig:simulated_lasso_half}
        \end{subfigure}

        \begin{subfigure}[t]{0.48\textwidth}
                \includegraphics[width=\textwidth]{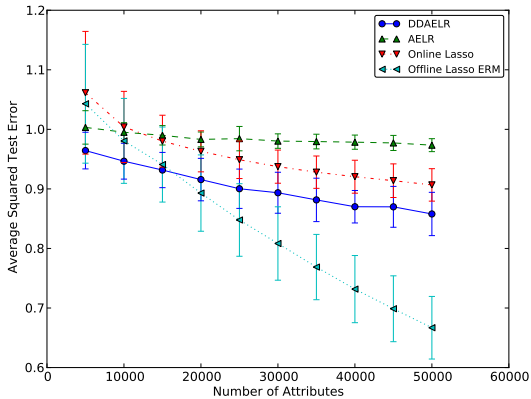}
                \caption{The second moments were chosen by a power law with an exponent of $\alpha=-1$, which results in an improvement factor of $\rho_\text{lasso}=0.014$.}
                \label{fig:simulated_lasso_one}
        \end{subfigure}
        \quad
        \begin{subfigure}[t]{0.48\textwidth}
                \includegraphics[width=\textwidth]{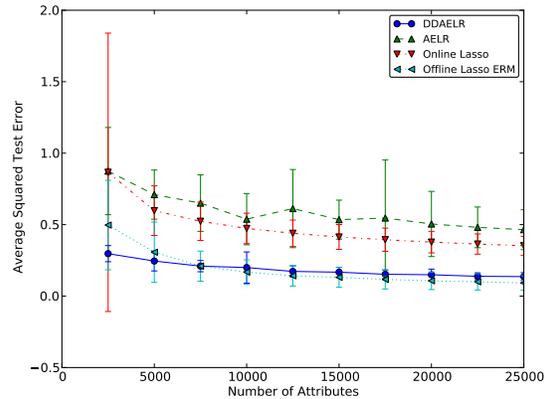}
                \caption{The second moments were chosen by a power law with an exponent of $\alpha=-2$, which results in an improvement factor of $\rho_\text{lasso}=$0.0033.}
                \label{fig:simulated_lasso_two}
        \end{subfigure}
        \caption{Test error for the algorithms with $k+1=5$ in the lasso scenario over simulated data with $d=500$.}
        \label{fig:simulated_lasso}
\end{figure}

The results for the lasso scenario that appear figure~\ref{fig:simulated_lasso} show the same behaviour, this time with respect to $\norm{\expectationD{\vec{x}^2}}{1}$ instead of $\norm{\expectationD{\vec{x}^2}}{\frac{1}{2}}$.

\subsection{MNIST Data Set}\label{sec:mnist}
In our next set of experiments, we choose to repeat the experiments in \cite{ohad,hazan} and use the popular MNIST data set. Each training example is a labeled $28\times28$ gray scale image of one hand-written digit. As in the original experiments, we have focused on the classification problem of distinguishing between the "3" digits (which we labeled -1) and the "5" digits (which we labeled +1). As in \cite{hazan},  we have used $k+1=57$ attributes for each training example in the ridge scenario and $k+1=5$ attributes in the lasso scenario. For this data set we have $d=784,\:\rho_\text{ridge}=0.45$ and $\rho_\text{lasso}=0.2$.

\begin{figure}[!htb]
\centering
\includegraphics[width=\textwidth]{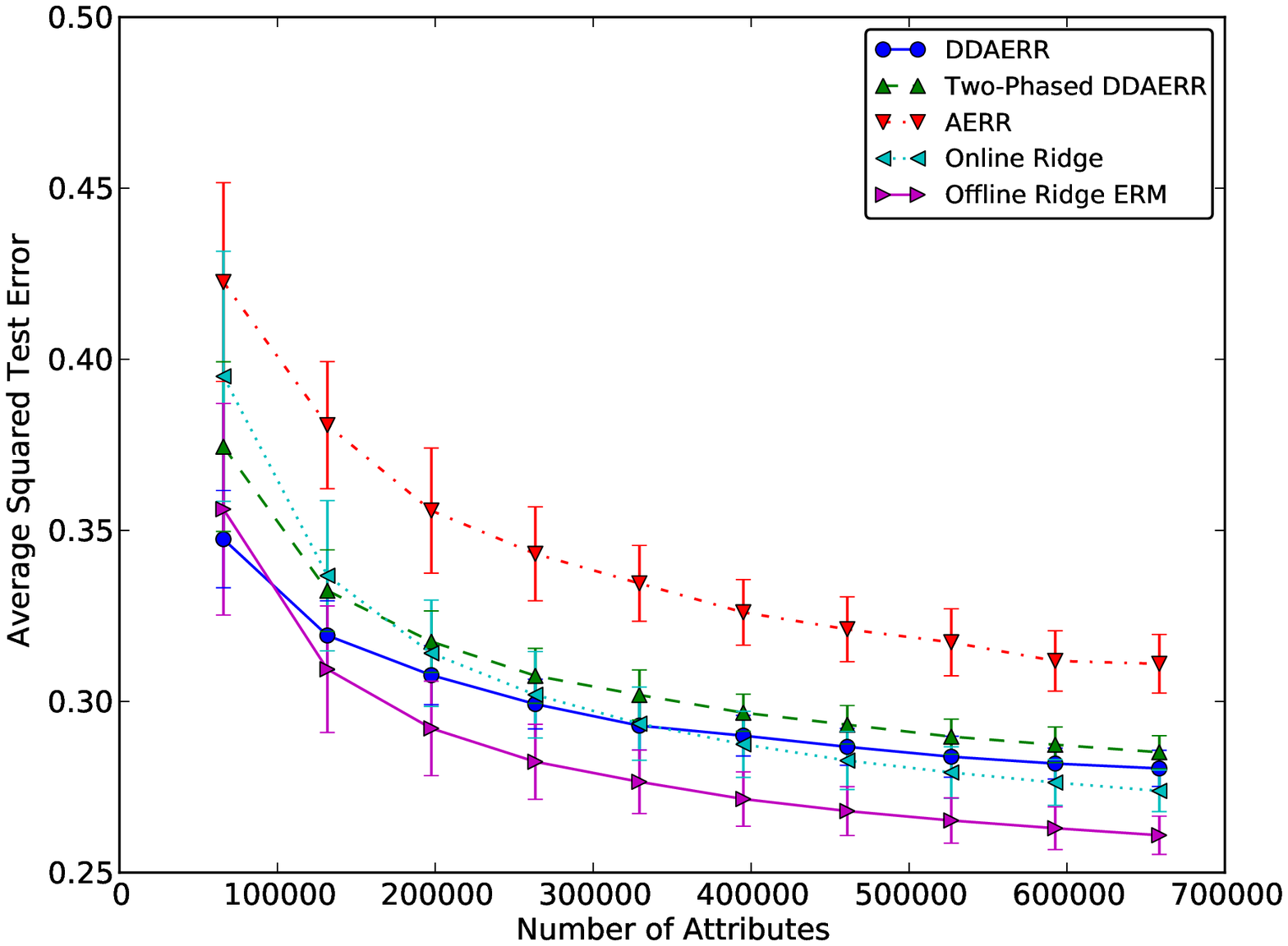}
\caption{Test error for the algorithms  with $k+1=57$ in the ridge scenario over the classification
task "3" vs. "5" in the MNIST data set.}
\label{fig:mnist_ridge}
\end{figure}

The results for the ridge scenario appear in figure~\ref{fig:mnist_ridge}: our DDAERR algorithm performs considerably better than the AERR algorithm, for all the training set sizes checked, in correspondence with the theory. Also, the DDAERR algorithm performs
similarly to the online ridge algorithm, and even better for a small total number
of examined attributes. This suggests that at least for a small number of total attributes,
our attribute efficient method is better than the full-information method. The offline ridge algorithm is still the best algorithm, because it can utilize
all attributes from each example thus reducing the variance, as well as use each example more than
once - privileges the attribute efficient algorithms
lack. The Two-Phased DDAERR algorithm performs between the AERR algorithm and the DDAERR algorithm,
and converges towards the DDAERR algorithm as the number of observed attributes
grows, as expected.

\begin{figure}[!htb]
\centering
\includegraphics[width=\textwidth]{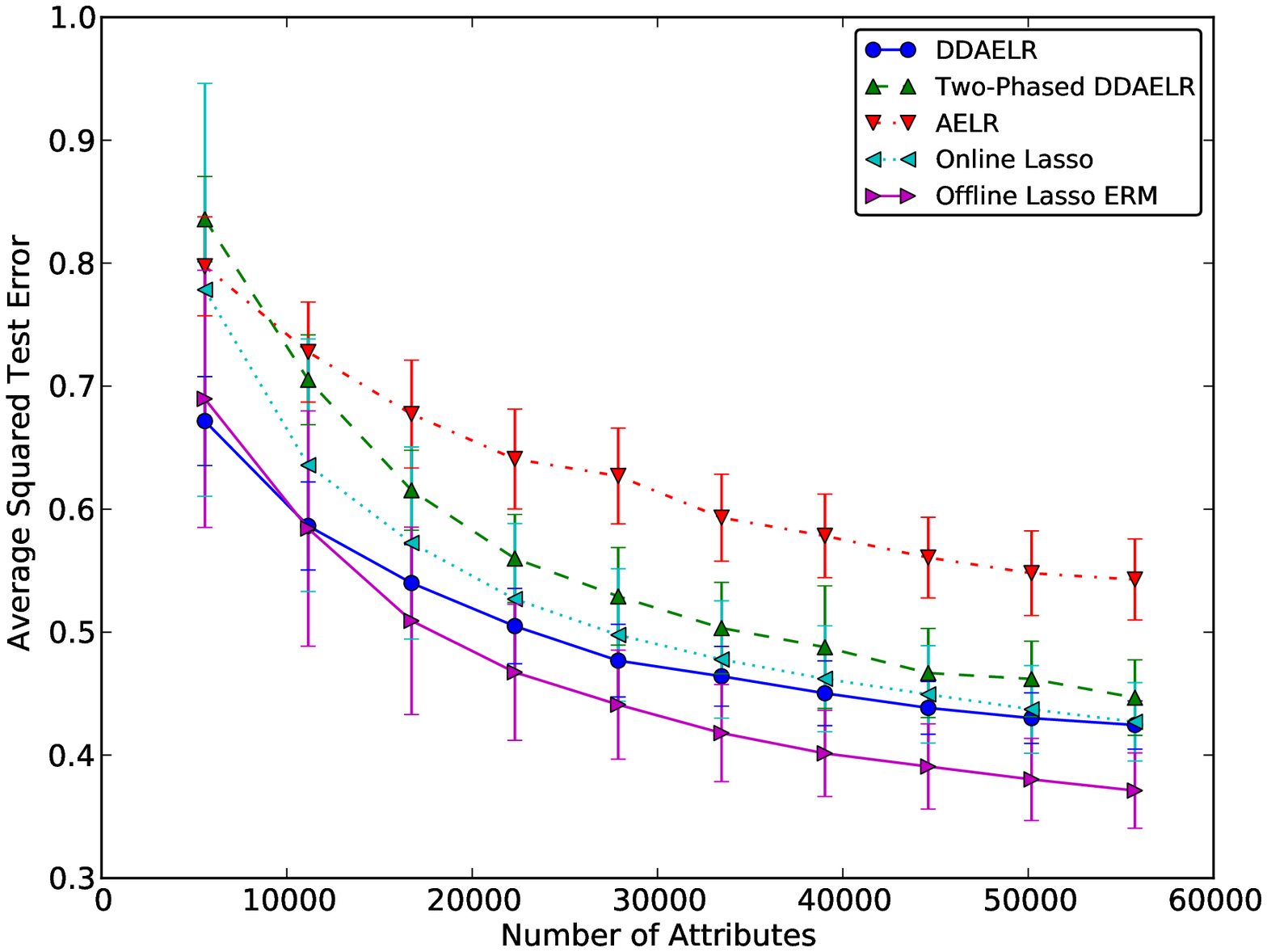}
\caption{Test error for the algorithms  with $k+1=5$ in the lasso scenario over the classification
task "3" vs. "5" in the MNIST data set.}
\label{fig:mnist_lasso}
\end{figure}
The results for the lasso scenario which appear in figure~\ref{fig:mnist_lasso} are similar: The DDAELR algorithms
performs considerably better than the AELR algorithm, and comparable with the online lasso algorithm, if not slightly better. It is interesting to note that the variability of the DDAELR algorithm is smaller than the variabilities of the other algorithms.  Also, this time it is much clearer that the Two-Phased DDAELR algorithm performs similarly to the AELR algorithm for a small amounts of examined attributes, and converges to DDAELR as the number of examined attributes increases. 

\subsection{Covertype Data Set}
In our last set of experiments we used the Covertype data set which aims to predict the forest cover type i.e. the dominant species of tree, from cartographic variables. This data set is designed
for multi class classification, but we reduce it to a binary classification
by choosing one of the tree species and address the problem by regressing on the $-1$ and $+1$ labels.
For both the ridge and lasso scenarios, we use a budget of $k+1=5$. For this data set we have $d=54,\:\rho_\text{ridge}=0.49$ and $\rho_\text{lasso}=0.08$.

\begin{figure}[!htb]
\centering
\includegraphics[width=\textwidth]{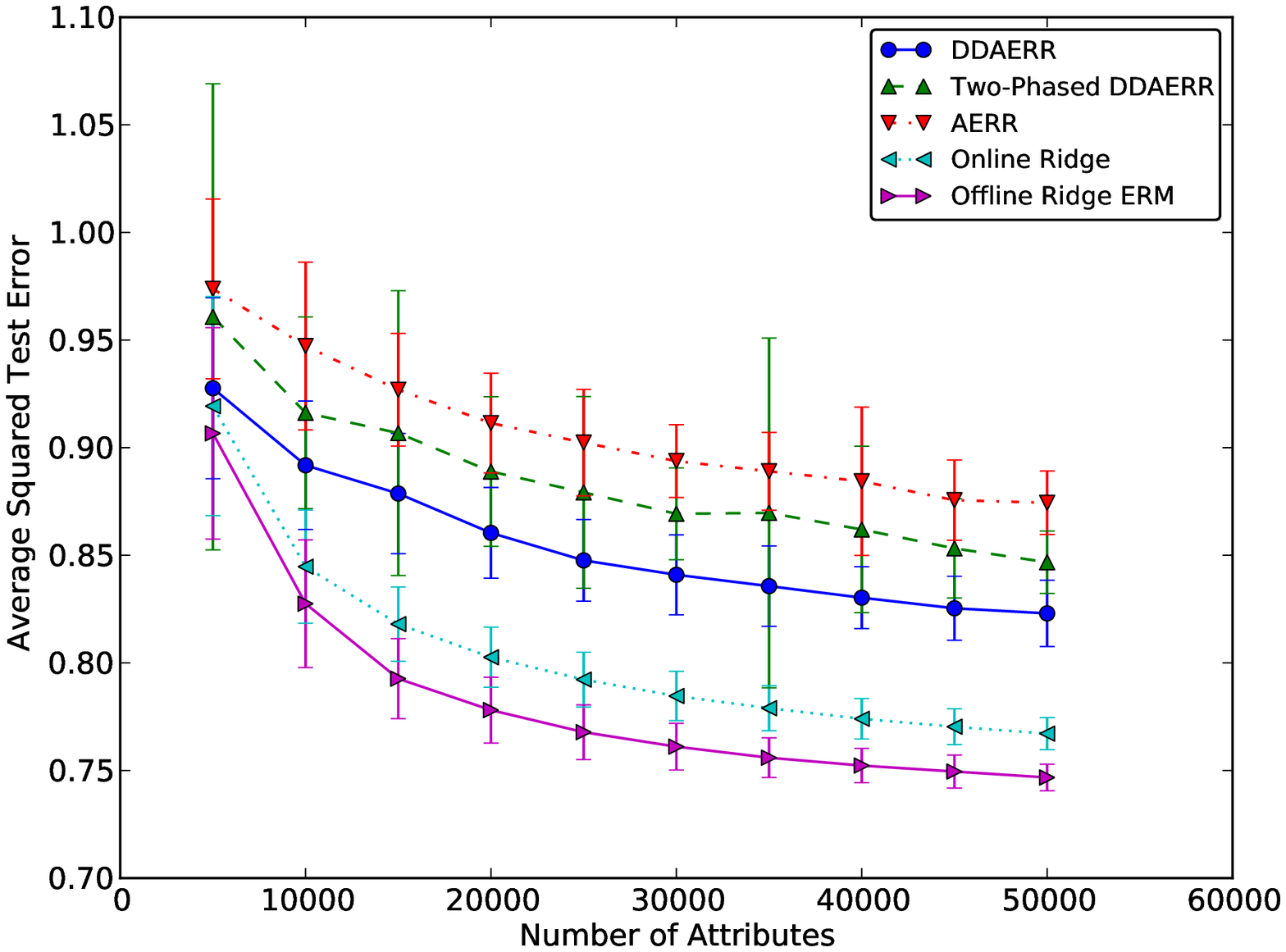}
\caption{Test error for the algorithms  with $k+1=5$ in the ridge scenario over the classification task in the Cover Type data set.}
\label{fig:cover_type_ridge}
\end{figure}

The results for the ridge scenario appear in figure~\ref{fig:cover_type_ridge}: Again, our DDAERR algorithm performs considerably better than the AERR algorithm. Also, the DDAERR algorithm performs
similarly to the online ridge algorithm for a small number
of examined attributes. The Two-Phased DDAERR algorithm performs between the AERR algorithm and the DDAERR algorithm, and given a larger training set will probably converge towards the DDAERR algorithm as the number of examined attributes grow. This time, however, the full-information ridge algorithms outperform the attribute efficient ones.

\begin{figure}[!htb]
\centering
\includegraphics[width=\textwidth]{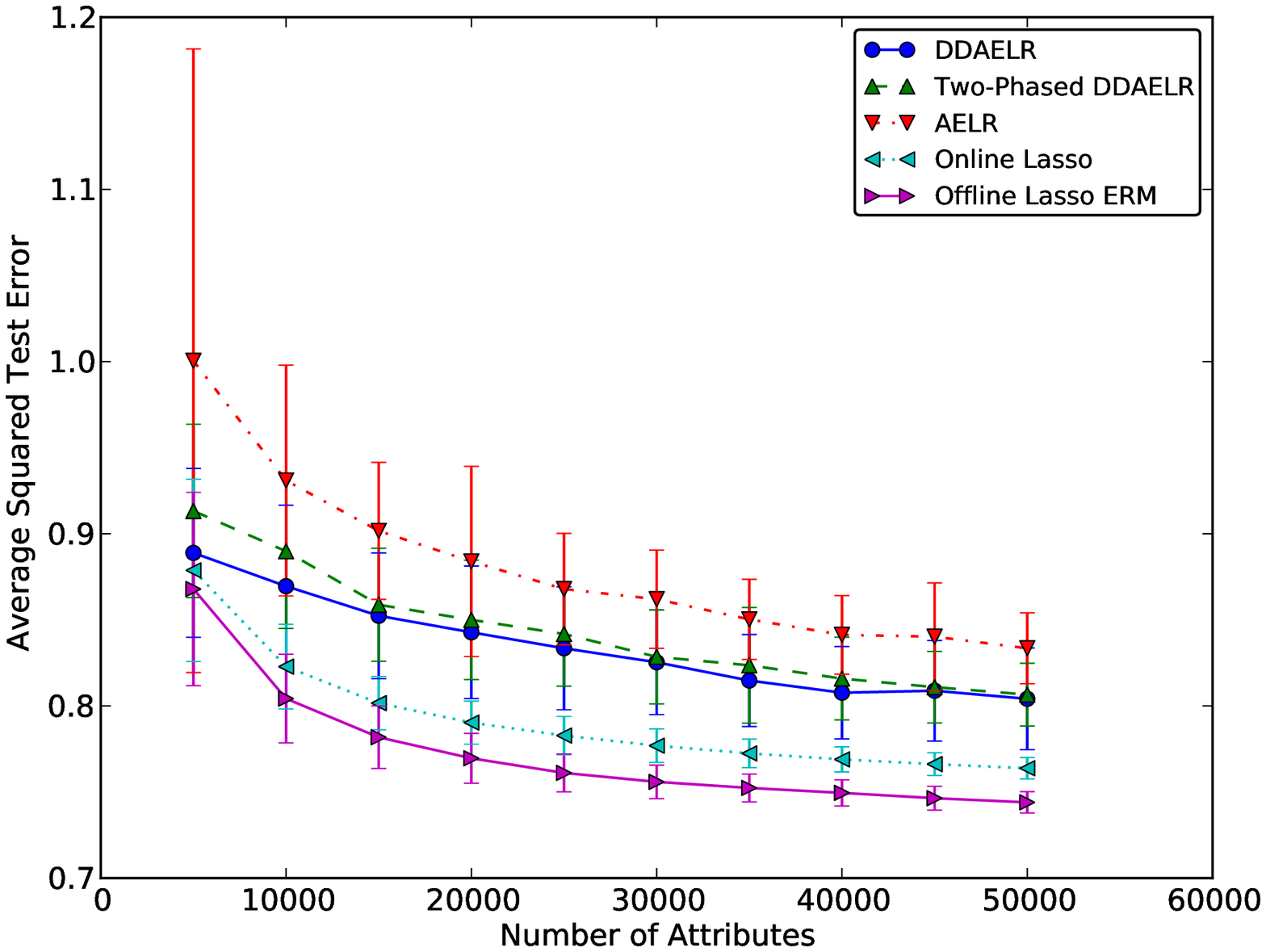}
\caption{Test error for the algorithms  with $k+1=5$ in the lasso scenario over the classification task in the Cover Type data set.}
\label{fig:cover_type_lasso}
\end{figure}

The results for the lasso scenario which appear in figure~\ref{fig:cover_type_lasso} are similar: The DDAELR algorithm
performs better than the AELR algorithm. Also, the Two-Phased DDAELR algorithm performs between the AELR and DDAELR algorithms and converges towards the DDAELR algorithm, as the number of attributes grows.
For a small number of examined attributes, the DDAELR algorithm performs similarly
to the online lasso algorithm, and with a smaller variability, but as the number of examined attributes grow, the algorithms drift apart.

\section{Summary and Extensions}\label{sec:discussion}
In this paper, we studied the attribute efficient
local budget setting and developed efficient linear algorithms for the ridge and lasso regression scenarios. Our algorithms utilize the geometry
of the data distribution, and are able to achieve data-dependent improvement factors
for the excess risk bound over the state of the art, which can be large as $\BigO{\sqrt{d}}$. We proved our claims analytically as well as demonstrated them empirically over several data sets.

Our method, even though applied here only for regression scenarios, is quite general, and potentially will be effective in other partial-information learning problems.

There are several possible directions for further research: First, as our algorithm bounds hold only in expectation, the question of how to extend them to hold with high probability, arises.
Second, while our work focuses on learning from i.i.d. stochastic data, it
is interesting to understand whether analogous results can hold in the online
learning scenario \cite{shalev}, where the data is not assumed to be stochastic. In addition, understanding the exact connection between the attribute efficient algorithms and the adaptive methods may lead to an additional improvement in our algorithms.
Another direction for future research may be to use the geometry of
the optimal linear predictor besides the geometry of the data. For example, if for some $i$, $\vec{w}_{t,i}$ is small, perhaps the learner should sample it less.  
Finally, proving data-dependent lower bounds may complement our results, or show additional room for improvement.  

\section*{Acknowledgements}
This research was partially supported by an Israel Science Foundation Grant (425/13) and an FP7 Marie Curie CIG grant.

\bibliography{bibfile}
\bibliographystyle{unsrt}

\appendix
\section{AdaGrad}\label{app:adagrad}
Our data-dependent results rely on the data having different moments along different directions. Such a situation has also been used to improve gradient descent methods in the full information case. The idea is that in such cases, it may be more beneficial to use a different learning rate along each coordinate, rather than a single global rate.

In particular, AdaGrad \cite{duchi} is a popular approach along these lines. Instead of using a global step size, it uses a different step size for each coordinate, defined as

\begin{equation}
\eta_{t,i}=\BigO{\frac{1}{\sqrt{\sum{_{s=1}^t\widetilde{g}_{s,i}^2}}}}
\end{equation}
for the $i$-th. The algorithm can be considered as running a separate copy of OGD for each attribute, with  a suitable learning rate.

The algorithm is never worse than using a global learning rate, and may be better by a factor dependent on the variability of the magnitude of the gradient estimates across attributes, similarly to our data-dependent algorithms. Therefore, a question arises of what is the connection between our algorithm and AdaGrad, and whether they interfere or support each other.

In this paper we do not theoretically analyze an adaptive gradient version of our algorithms, but provide and discuss simulation results.

We run two simulations. The first, on the artificial data set from section~\ref{sec:artificial} with $d=500$, a decaying exponent of $\alpha=-2$ and $\rho_\text{ridge}=0.05$. As in the original experiment, we have used $k+1=5$. The second experiment, on the subset of the MNIST data set from section~\ref{sec:mnist} with $d=784$ and $\rho_\text{ridge}=0.45$.
As in the original experiment, we have used $k+1=57$.

\begin{figure}[!htb]
\centering
\includegraphics[width=\textwidth]{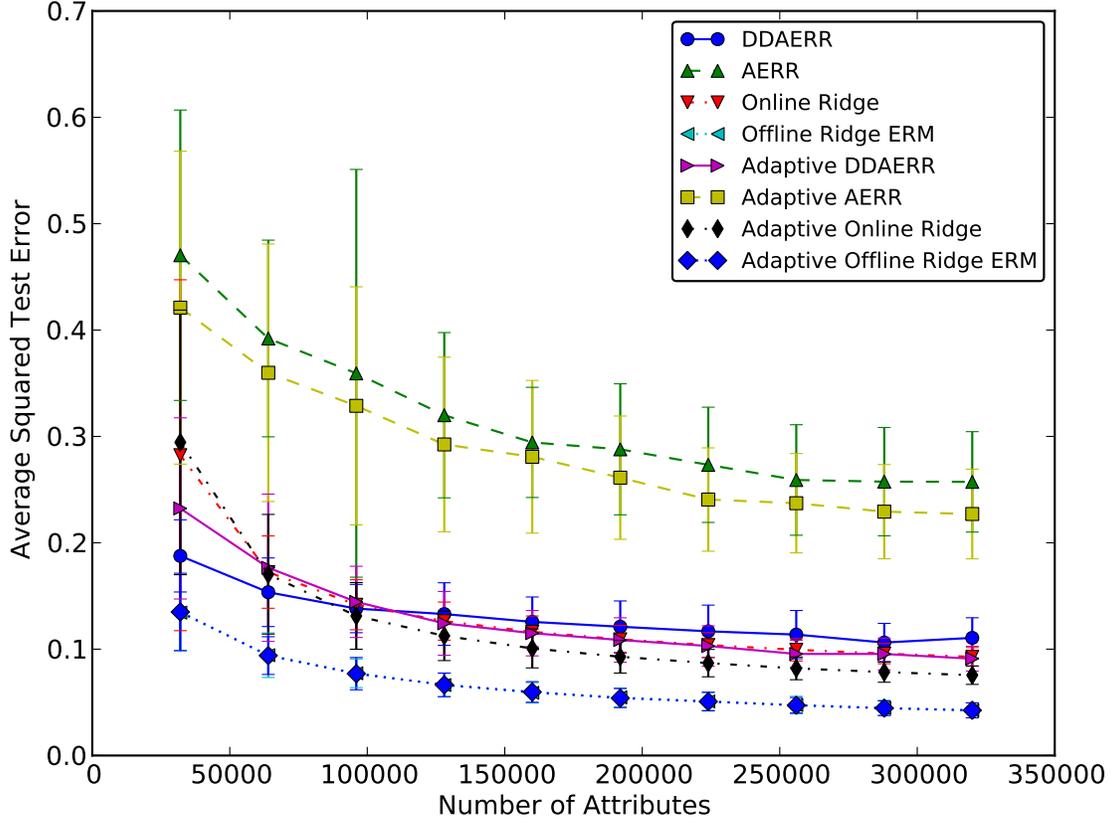}
\caption{Test error for the algorithms with $k+1=5$ in the ridge scenario over simulated data with $d=500$, a decaying exponent of $\alpha=-2$ and $\rho_\text{ridge}=0.05$.}
\label{fig:adaptive_ridge_two}
\end{figure}

The results for the simulated data set appear in figure~\ref{fig:adaptive_ridge_two}: In all cases except the offline ERM, the adaptive algorithm performs slightly better than the corresponding non-adaptive algorithm. For the offline ERM algorithm, both preform the same. 

\begin{figure}[!htb]
\centering
\includegraphics[width=\textwidth]{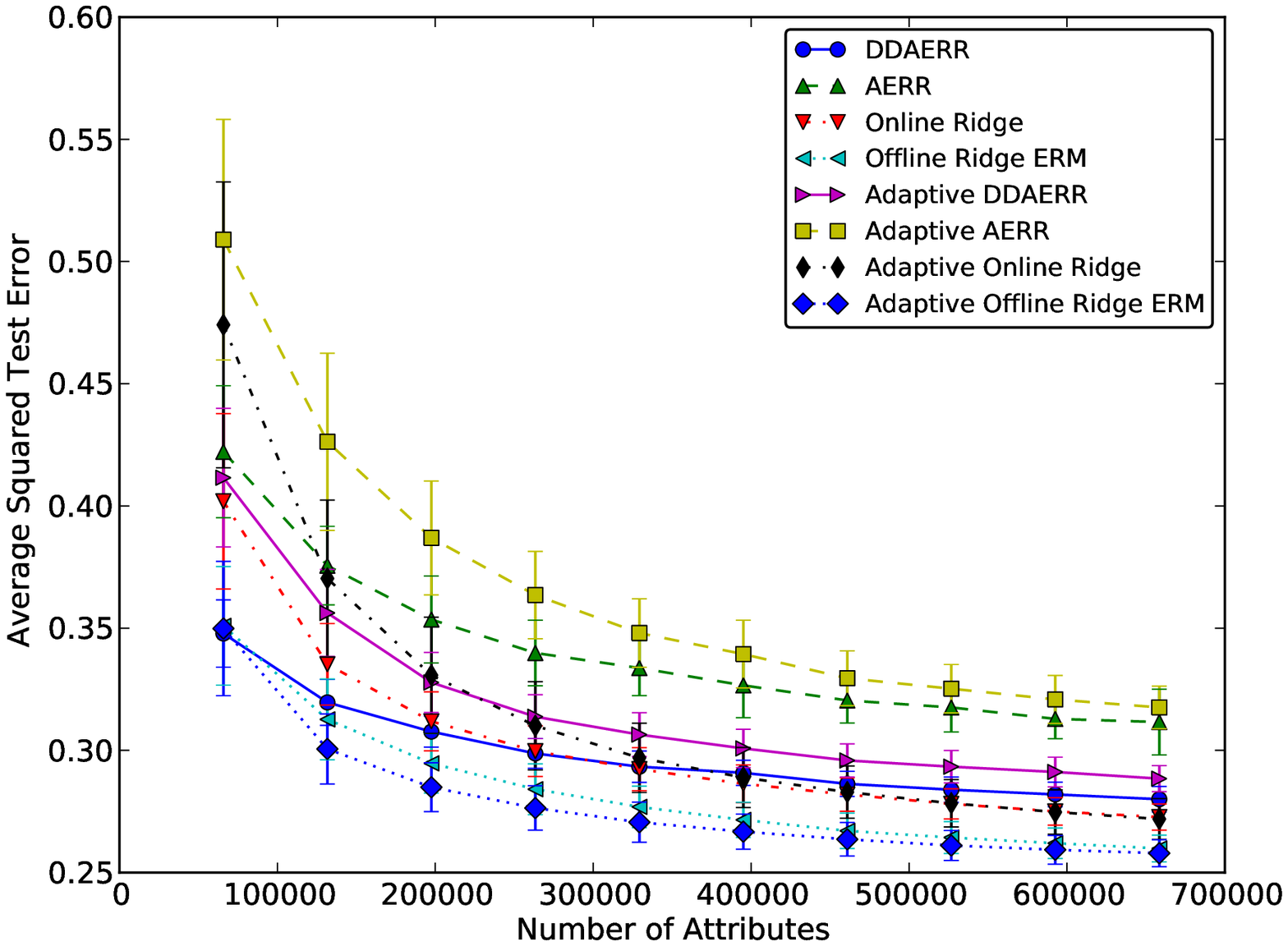}
\caption{Test error for the algorithms  with $k+1=57$ in the ridge scenario over the classification
task "3" vs. "5" in the MNIST data set.}
\label{fig:adaptive_mnist_ridge}
\end{figure}

The results for the MNIST data set, which appear in figure~\ref{fig:adaptive_mnist_ridge}, are different: The adaptive version of the full-information algorithms performs slightly better than the corresponding full-information algorithms. However, the adaptive version of the attribute efficient algorithms performs worse than their non-adaptive version.

The bottom line of these simulations is not conclusive: In some scenarios, the adaptive method improve the results whereas in others, it degrades them. Further research is required in order to understand better the connection between the attribute efficient scenario and the adaptive gradient descent improvements.

\section{Proofs}
\subsection{Proof of Theorem~\ref{thm:maingaerr}}\label{app:thm:maingaerrproof}
We use the standard analysis of the OGD algorithm. The expected risk bound of the it is stated in the following lemma.
\begin{lemma}[Zinkevich, 2003]\label{lm:zink}
For any $\norm{\vec{w}^{*}}{} \leq B$, we have 
\begin{equation}\label{eq:zink}
\sum_{t=1}^{m}\widetilde{\vec{g}}_{t}^{T}\left(\vec{w}_{t}-\vec{w}^{*}\right) \leq \frac{2B^{2}}{\eta} + \frac{\eta}{2}\sum_{t=1}^{m}\norm{\widetilde{\vec{g}_t}}{2}^{2}.
\end{equation}
\end{lemma}
The proof can be found in \cite{zink}.

To use this lemma, first we need to prove that the GAERR algorithm actually performs stochastic gradient descent. To show this, it is enough to prove that $\widetilde{\vec{g}}_{t}$ is an unbiased estimator of the gradient, as stated in the next lemma:
\begin{lemma}\label{lm:ridgemain}
The vector $\widetilde{\vec{g}}_{t}$ is an unbiased estimator of the gradient $\vec{g}_{t}=\left(\vec{w}_{t}^{T}\vec{x}_{t}-y_{t}\right)\vec{x}_{t}$, that is $\expectationA{\widetilde{\vec{g}}_{t}}=\vec{g}_{t}$. 
\end{lemma}

Now, we can take the expectation of  Lemma~\ref{lm:zink} with respect to
the randomization of the algorithm and the data distribution, and using Lemma~\ref{lm:ridgemain} we have
\begin{equation*}
\expectationDA{\sum_{t=1}^{m}\vec{g}_{t}^{T}\left(\vec{w}_{t}-\vec{w}^{*}\right)} \leq \frac{2B^{2}}{\eta} + \frac{\eta}{2}G^{2}m.
\end{equation*}
On the other hand, the convexity of $\ell$ gives $\ell_{t}\left(\vec{w}_{t}\right)-\ell_{t}\left(\vec{w}^{*}\right)\leq
\vec{g}^{T}\left(\vec{w}_{t}-\vec{w}^{*}\right)$. Together with the above
we have
\begin{equation*}
\expectationDA{\frac{1}{m}\sum_{t=1}^{m}\ell_t\left(\vec{w}_{t}\right)}\leq
\expectationDA{\frac{1}{m}\sum_{t=1}^{m}\ell_t\left(\vec{w}^*\right)} + \frac{2B^{2}}{\eta m} + \frac{\eta}{2}G^{2},
\end{equation*}
or
\begin{equation*}
\expectationDA{\frac{1}{m}\sum_{t=1}^{m}L_{\mathcal{D}}\left(\vec{w}_{t}\right)}\leq
L_{\mathcal{D}}\left(\vec{w}^*\right) + \frac{2B^{2}}{\eta m} + \frac{\eta}{2}G^{2},
\end{equation*}
Using the convexity of $L_{\mathcal{D}}$ and Jensen's
inequality, the theorem follows.

\begin{proof}[Proof of Lemma~\ref{lm:ridgemain}]
First, it is straightforward to see that $\expectationA{\widetilde{\vec{x}}_{t,r}} =\vec{x}_{t}$ for all $r$ thus also $\expectationA{\widetilde{\vec{x}}_{t}} =\vec{x}_{t}$.
Also, a simple calculation shows that

\begin{equation*}
\expectationA{\widetilde{\phi_{t}}}=\sum_{j=1}^{d}p_{j}\left(\frac{w_{t,j}}{p_{j}}\vec{x}_{t}\left[j\right]-y_{t}\right)=\vec{w}_{t}^{T}\vec{x}_{t}-y_{t}.
\end{equation*}
Since $\widetilde{\vec{x}}_{t}$ and $\widetilde{\phi_{t}}$ are independent given $\vec{x}_t$, we obtain that $\expectationA{\widetilde{\vec{g}}_{t}}=\left(\vec{w}_{t}^{T}\vec{x}_{t}-y_{t}\right)\cdot\vec{x}_{t}$, which is the required gradient.
\end{proof}

\subsection{Proof of Lemma~\ref{lm:twonormbound}}\label{app:twonormboundproof}
From the definition of $\widetilde{\vec{x}}_{t}$ in equation~$\eqref{eq:ridgeestimator}$,
\begin{align*}
\expectationDA{\norm{\widetilde{\vec{x}}_{t}}{2}^2}
&=\frac{1}{k^2}\expectationDA{\norm{\sum_{r=1}^k\widetilde{\vec{x}}_{t,r}}{2}^2}\\
&=\frac{1}{k^2}\sum_{r=1}^k\expectationDA{\norm{\widetilde{\vec{x}}_{t,r}}{2}^2}+\frac{1}{k^2}\sum_{r=1}^k\sum_{s\neq r}^k\expectationDA{\left<\widetilde{\vec{x}}_{t,r},\widetilde{\vec{x}}_{t,s}\right>}.
\end{align*}
Since $\expectationDA{\widetilde{\vec{x}}_{t,r}}=\expectationD{\vec{x}}$ and $\widetilde{\vec{x}}_{t,r}$ and $\widetilde{\vec{x}}_{t,s}$ are independent of each other, we finally have
\begin{align*}
\expectationDA{\norm{\widetilde{\vec{x}}_{t}}{2}^2}
=\frac{1}{k}\expectationDA{\norm{\widetilde{\vec{x}}_{t,r}}{2}^2}+\frac{k^2-k}{k^2}\expectationD{\norm{\vec{x}}{2}}^2
=\frac{1}{k}\expectationDA{\norm{\widetilde{\vec{x}}_{t,r}}{2}^2}+\frac{k-1}{k}\expectationD{\norm{\vec{x}}{2}}^2.
\end{align*}

\subsection{Proof of Lemma~\ref{lm:twonormboundphi}}\label{app:twonormboundphiproof}
Recalling $\left|y_{t}\right|\leq B$ and using the inequality $\left(a-b\right)^2\leq2\left(a^2+b^2\right)$, by a straightforward calculation
 we obtain
\begin{align*}
\expectationDA{\widetilde{\phi_{t}}^2}
&=\expectationDA{\left(\frac{w_{t,j}}{p_j}\vec{x}_{t}\left[j_{t}\right]-y_{t}\right)^2}\\ &\leq2\expectationDA{\left(\frac{w_{t,j}}{p_{j}}\vec{x}_{t}\left[j_{t}\right]\right)^{2}+y_{t}^{2}}\\
&\leq2\sum_{j=1}^{d}\frac{1}{p_j}w_{t,j}^{2}\expectationD{\vec{x}_{j}^{2}}+2B^2\\
&=2\norm{\vec{w}_t}{2}^2\expectationD{\norm{\vec{x}}{2}^2}+2B^2\\
&\leq4B^2.
\end{align*}

\subsection{Proof of Lemma~\ref{lm:ridgenotworse}}\label{app:ridgenotworseproof}
First, we state a simple probabilistic lemma that will be used to bound the our estimates for the second moment of the attributes.

\begin{lemma}\label{lm:rv}
Let $Z_1,Z_2,...,Z_n$ be i.i.d random variables. $Z_i\in\left[0,1\right]$. Let $\hat{\mathbb{E}}\left[Z\right]=\frac{1}{n}\sum_{i=1}^{n}Z_i$ be their average. Then, with probability $\geq 1-\delta$
\begin{equation*}
\hat{\mathbb{E}}\left[Z\right] \leq 2\mathbb{E}\left[Z\right]  + \frac{7\log{\frac{1}{\delta}}}{6n}.
\end{equation*}
Also, with probability $\geq 1-\delta$
\begin{equation*}
\hat{\mathbb{E}}\left[Z\right] \geq \frac{1}{2}\mathbb{E}\left[Z\right]  - \frac{5\log{\frac{1}{\delta}}}{3n}.
\end{equation*}
\end{lemma}

We prefer to use this lemma rather than the standard Bernstein inequality because we are
interested in a fast convergence rate of $\frac{1}{n}$, and are willing to pay the price of
the additional constant factor.

To prove our lemma, we use the definition of $\norm{\widetilde{\vec{x}}_{t,r}}{2}^{2}$,
\begin{equation*}
\expectationDAphase{2}{\norm{\widetilde{\vec{x}}_{t,r}}{2}^{2}}=
\expectationDAphase{2}{\widetilde{\vec{x}}_{t,r}\left[i_{t,r}\right]^{2}}=
\sum_{i=1}^{d}\frac{1}{q_i}\expectationD{x_i^{2}}=
\sum_{j=1}^{d}\sqrt{A\left[j\right]+\frac{13}{6}\epsilon}\sum_{i=1}^{d}\frac{\expectationD{x_{i}^{2}}}{\sqrt{A\left[i\right]+\frac{13}{6}\epsilon}}.
\end{equation*}

For all $i\in\left[d\right]$ let $T_{i}$ be a random variable describing the amount of times the algorithm sampled the $i$-th attribute in the first phase. For every realization $t_i$ of $T_i$, since $T_i$ and the samples themselves are independent, we can use Lemma~\ref{lm:rv} and by the union bound have that with probability
larger than $1-\delta$, $\vec{A}\left[i\right] \leq 2\expectationD{x_{i}^{2}}+\frac{7}{6}\expectationAphase{1}{\epsilon_{i}}$,
and $\vec{A}\left[i\right] \geq \frac{1}{2}\expectationD{x_{i}^{2}}-\frac{5}{3}\expectationAphase{1}{\epsilon_{i}}$
where  $\epsilon_{i}=\frac{\log{\frac{2d}{\delta}}}{t_{i}}$. Clearly, $\expectationAphase{1}{T_i}=\frac{\left(k+1\right)m_1}{d}$, and using the convexity of $f\left(x\right)=\frac{1}{x}$ we have  $\expectationAphase{1}{\epsilon_{i}}\geq \frac{d\log{\frac{2d}{\delta}}}{\left(k+1\right)m_1} = \epsilon$. Therefore,
with probability $\geq1-\delta$ over the first phase, we have

\begin{equation}\label{eq:smallbounds}
    \left\{
          \begin{array}{ll}
              \vec{A}\left[i\right] \leq 2\expectationD{x_i^2}+\frac{7}{6}\epsilon\\
              \vec{A}\left[i\right] \geq \frac{1}{2}\expectationD{x_i^2}-\frac{5}{3}\epsilon.
          \end{array}
    \right.
\end{equation}Note that these equations also hold trivially for any $\epsilon\geq1$ as with probability 1 we have $x_i^2\leq1$ for all $i\in\left[d\right]$.

Now we can continue and see,
\begin{align*}
\expectationDAphase{2}{\norm{\widetilde{\vec{x}}_{t,r}}{2}^{2}}
&\leq\sum_{j=1}^{d}\sqrt{2\expectationD{x_{i}^{2}}+\frac{7}{6}\epsilon+\frac{13}{6}\epsilon}\sum_{i=1}^{d}\frac{\expectationD{x_{i}^{2}}}{\sqrt{\frac{1}{2}\expectationD{x_{i}^{2}}-\frac{5}{3}\epsilon+\frac{13}{6}\epsilon}}\\
&=\sum_{j=1}^{d}\sqrt{2\left(\expectationD{x_{i}^{2}}+\frac{5}{3}\epsilon\right)}\sum_{i=1}^{d}\frac{\expectationD{x_{i}^{2}}}{\sqrt{\frac{1}{2}\left(\expectationD{x_{i}^{2}}+\epsilon\right)}}\\
&=2\sum_{j=1}^{d}\sqrt{\expectationD{x_{i}^{2}}+\frac{5}{3}\epsilon}\sum_{i=1}^{d}\frac{\expectationD{x_{i}^{2}}}{\sqrt{\expectationD{x_{i}^{2}}+\epsilon}}.
\end{align*}

We shall bound this value in two ways. For the first part of the lemma, we have
\begin{align*}
\expectationDAphase{2}{\norm{\widetilde{\vec{x}}_{t,r}}{2}^{2}}
&\leq2\sum_{j=1}^{d}\sqrt{\expectationD{x_{j}^{2}}+\frac{5}{3}\epsilon}\sum_{i=1}^{d}\frac{\expectationD{x_{i}^{2}}}{\sqrt{\expectationD{x_{i}^{2}}+\epsilon}}\\
&\leq2\sum_{j=1}^{d}\sqrt{\expectationD{x_{j}^{2}}}\sum_{i=1}^{d}\frac{\expectationD{x_{i}^{2}}}{\sqrt{\expectationD{x_{i}^{2}}+\epsilon}} + 2\sum_{j=1}^{d}\sqrt{\frac{5}{3}\epsilon}\sum_{i=1}^{d}\frac{\expectationD{x_{i}^{2}}}{\sqrt{\expectationD{x_{i}^{2}}+\epsilon}}\\
&\leq2\sum_{j=1}^{d}\sqrt{\expectationD{x_{j}^{2}}}\sum_{i=1}^{d}\frac{\expectationD{x_{i}^{2}}}{\sqrt{\expectationD{x_{i}^{2}}}} + 2\sum_{j=1}^{d}\sqrt{\frac{5}{3}\epsilon}\sum_{i=1}^{d}\frac{\expectationD{x_{i}^{2}}}{\sqrt{\epsilon}}\\
&\leq2\norm{\expectationD{\vec{x}^{2}}}{\frac{1}{2}}+ 2d\sqrt{\frac{5}{3}}\sum_{i=1}^{d}\expectationD{x_{i}^{2}}\\
&\leq2\norm{\expectationD{\vec{x}^{2}}}{\frac{1}{2}}+ 2\sqrt{\frac{5}{3}}d\\ &\leq5d.
\end{align*}
As this bound is independent of $\epsilon$, it actually holds with probability $1$
over the first phase.

For the second part of the lemma, we have
\begin{align*}
\expectationDAphase{2}{\norm{\widetilde{\vec{x}}_{t,r}}{2}^{2}}  
&\leq 2\sum_{j=1}^{d}\sqrt{\expectationD{x_{j}^{2}}+\frac{5}{3}\epsilon}\sum_{i=1}^{d}\frac{\expectationD{x_{i}^{2}}}{\sqrt{\expectationD{x_{i}^{2}}+\epsilon}}\\
&\leq 2\sum_{j=1}^{d}\sqrt{\expectationD{x_{j}^{2}}+\frac{5}{3}\epsilon}\sum_{i=1}^{d}\frac{\expectationD{x_{i}^{2}}}{\sqrt{\expectationD{x_{i}^{2}}}}\\
&\leq 2\sum_{j=1}^{d}\sqrt{\expectationD{x_{j}^{2}}}\sum_{i=1}^{d}\frac{\expectationD{x_{i}^{2}}}{\sqrt{\expectationD{x_{i}^{2}}}}+
  2\sum_{j=1}^{d}\sqrt{\frac{5}{3}\epsilon}\sum_{i=1}^{d}\frac{\expectationD{x_{i}^{2}}}{\sqrt{\expectationD{x_{i}^{2}}}}\\
&\leq 2\norm{\expectationD{\vec{x}^{2}}}{\frac{1}{2}}+ 2\sqrt{\frac{5}{3}}d\sqrt{\norm{\expectationD{\vec{x}^{2}}}{\frac{1}{2}}}\sqrt{\epsilon}.
\end{align*}

\begin{proof}[Proof of Lemma~\ref{lm:rv}]
Let us denote the variance of $Z$ by $\sigma^2=\mathbb{E}\left[Z^2\right]-\mathbb{E}\left[Z\right]^2$. Using Bernstein's inequality, with probability $\geq 1-\delta$,
we have\begin{equation*}
\hat{\mathbb{E}}\left[Z\right] \leq \mathbb{E}\left[Z\right] + \sqrt{\frac{2\sigma^2\log{\frac{1}{\delta}}}{n}} + \frac{2\log{\frac{1}{\delta}}}{3n}.
\end{equation*}
Using $Z_i\in\left[0,1\right]$, we obtain $\sigma^2=\mathbb{E}\left[Z^2\right]-\mathbb{E}\left[Z\right]^2 \ \leq  \mathbb{E}\left[Z^2\right] \leq \mathbb{E}\left[Z\right]$. Plugging back in the expression for $\hat{\mathbb{E}}\left[Z\right]$,
\begin{equation*}
\hat{\mathbb{E}}\left[Z\right] \leq \mathbb{E}\left[Z\right] + \sqrt{\frac{2\mathbb{E}\left[Z\right]\log{\frac{1}{\delta}}}{n}} + \frac{2\log{\frac{1}{\delta}}}{3n}.
\end{equation*}
Using the fact that the geometric mean is smaller or equal to the arithmetic mean, we have
\begin{equation*}
\hat{\mathbb{E}}\left[Z\right] \leq \mathbb{E}\left[Z\right] + \frac{2\mathbb{E}\left[Z\right]}{2} + \frac{\log{\frac{1}{\delta}}}{2n} + \frac{2\log{\frac{1}{\delta}}}{3n}
\end{equation*}
or,
\begin{equation*}
\hat{\mathbb{E}}\left[Z\right] \leq 2\mathbb{E}\left[Z\right]  + \frac{7\log{\frac{1}{\delta}}}{6n},
\end{equation*}
which concludes the first part of the proof.

Similarly, using Bernstein's inequality again, with probability $\geq 1-\delta$,
we have
\begin{equation*}
\hat{\mathbb{E}}\left[Z\right] \geq \mathbb{E}\left[Z\right] - \sqrt{\frac{2\sigma^2\log{\frac{1}{\delta}}}{n}} - \frac{2\log{\frac{1}{\delta}}}{3n}.
\end{equation*}
Using $\sigma^2 \leq \mathbb{E}\left[Z\right]$, this turns to
\begin{equation*}
\hat{\mathbb{E}}\left[Z\right] \geq \mathbb{E}\left[Z\right] - \sqrt{\frac{2\mathbb{E}\left[Z\right]\log{\frac{1}{\delta}}}{n}} - \frac{2\log{\frac{1}{\delta}}}{3n}.
\end{equation*}
Again using the fact that the geometric mean is smaller or equal to the arithmetic mean, we have
\begin{equation*}
\hat{\mathbb{E}}\left[Z\right] \geq \mathbb{E}\left[Z\right] - \frac{\mathbb{E}\left[Z\right]}{2} - \frac{2\log{\frac{1}{\delta}}}{2n} - \frac{2\log{\frac{1}{\delta}}}{3n}
\end{equation*}
or,
\begin{equation*}
\hat{\mathbb{E}}\left[Z\right] \geq \frac{1}{2}\mathbb{E}\left[Z\right]  - \frac{5\log{\frac{1}{\delta}}}{3n},
\end{equation*}
 which concludes the proof.
\end{proof}

\subsection{Proof of Lemma~\ref{lm:onlineridgenotworse}}\label{app:onlineridgenotworse}
First, using Theorem~\ref{thm:maingaerr} on the second phase of the algorithm, we have
\begin{equation}\label{eq:onlineDDAERRbound}
\expectationDAphase{2}{L_{\mathcal{D}}\left(\bar{\vec{w}}\right)}-
L_{\mathcal{D}}\left(\vec{w^{*}}\right)\leq\frac{2B^{2}}{\eta m_{2}} + \frac{\eta}{2}G^{2}.
\end{equation}
Now we use the first part of Lemma~\ref{lm:ridgenotworse}, plug it into Lemma~\ref{lm:gaerrmain} and obtain that with probability $1$, we have
 $G^2 \leq 4B^{2}\left(\frac{5d}{k}+1\right)\leq24B^2\frac{d}{k}$.
Plugging $\eta=\sqrt{\frac{k}{6dm_{2}}}$ into equation \eqref{eq:onlineDDAERRbound} finishes the proof.

\subsection{Proof of Lemma~\ref{lm:onlineridgepartial}}\label{app:onlineridgepartialproof}
We use second part of Lemma~\ref{lm:ridgenotworse}, plug it into Lemma~\ref{lm:gaerrmain} and obtain that with probability $\geq1-\delta$, we have $G^2 \leq 4B^{2}\left(\frac{2}{k}\norm{\expectationD{\vec{x}^{2}}}{\frac{1}{2}}+ \frac{2}{k}\sqrt{\frac{5}{3}}d\sqrt{\norm{\expectationD{\vec{x}^{2}}}{\frac{1}{2}}}\sqrt{\epsilon}+1\right)$.
We denote $\widehat{G^2}=4B^{2}\left(\frac{2}{k}H+\frac{2}{k}\sqrt{\frac{5}{3}}d\sqrt{H}\sqrt{\epsilon}+1\right)$. Since $H\geq\norm{\expectationD{\vec{x}^{2}}}{\frac{1}{2}}$ we have $G^2\leq\widehat{G^2}$. 
Plugging $\eta=\frac{2B}{\sqrt{\widehat{G^2}m_2}}=\frac{1}{\sqrt{m_2\left(\frac{2}{k}H+ \frac{2}{k}\sqrt{\frac{5}{3}}d\sqrt{H}\sqrt{\epsilon}+1\right)}}$ into equation \eqref{eq:onlineDDAERRbound},
we get

\begin{align*}
\expectationDAphase{2}{L_{\mathcal{D}}\left(\bar{\vec{w}}\right)}-L_{\mathcal{D}}\left(\vec{w}_{*}\right)
&\leq\frac{2B^{2}}{\eta m_2}+\frac{\eta}{2}G^{2}\\
&\leq\frac{2B^{2}}{\eta m_2}+\frac{\eta}{2}\widehat{G^2}\\
&\leq\frac{2B}{\sqrt{m_2}}\sqrt{\widehat{G^2}}\\
&=\frac{4B^2}{\sqrt{m_2}}\sqrt{\frac{2}{k}H+\frac{2}{k}\sqrt{\frac{5}{3}}d\sqrt{H}\sqrt{\epsilon}+1}.
\end{align*}

\subsection{Proof of Lemma~\ref{lm:onlineridgepartialcombined}}\label{app:onlineridgepartialcombined}
First, we state a simple lemma that will allow us to combine two risk bounds,
each is achieved by a different value of $\eta$.
\begin{lemma}\label{lm:combineregrets}
Let $f\left(\eta\right)=\frac{A}{\eta} + \eta B G^{2}$ for some positive constants $A,B,G$, where $G\leq\min\left(G_1,G_2\right)$. Let  $\eta_i=\frac{1}{G_i}\sqrt{\frac{A}{B}}$ for $i=1,2$. Then $f\left(\max\left(\eta_1,\eta_2\right)\right)\leq\min\left(f\left(\eta_1\right),f\left(\eta_2\right)\right)$.
\end{lemma} 

By Lemma~\ref{lm:onlineridgenotworse} , using $\eta=\sqrt{\frac{k}{12dm_{2}}}$, we have
with probability $1,$
\begin{equation*}
\expectationDAphase{2}{L_{\mathcal{D}}\left(\bar{\vec{w}}\right)}-
L_{\mathcal{D}}\left(\vec{w}^{*}\right)\leq 4B^2\sqrt{\frac{6d}{km_2}}.
\end{equation*}
Similarly, by Lemma~\ref{lm:onlineridgepartial}, using 
$\eta=\frac{1}{\sqrt{m_2\left(\frac{2}{k}H+ \frac{2}{k}\sqrt{\frac{5}{3}}d\sqrt{H}\sqrt{\frac{d\log{\frac{2d}{\delta}}}{\left(k+1\right)m_1}}+1\right)}}$,
we have with probability $\geq1-\delta$, 
\begin{equation*}
\expectationDAphase{2}{L_{\mathcal{D}}\left(\bar{\vec{w}}\right)}-
L_{\mathcal{D}}\left(\vec{w}^{*}\right)\leq\frac{4B^2}{\sqrt{m_2}}\sqrt{\frac{2}{k}H+\frac{2}{k}\sqrt{\frac{5}{3}}d\sqrt{H}\sqrt{\frac{d\log{\frac{2d}{\delta}}}{\left(k+1\right)m_1}}+1}.
\end{equation*}
Using Theorem~\ref{thm:maingaerr}, the expected risk bound has the form of the function in Lemma~\ref{lm:combineregrets}, and the theorem follows directly.

\begin{proof}[Proof of Lemma~\ref{lm:combineregrets}]
Assume without loss of generality that $G_1\geq G_2$, therefore we also have
$\eta_2>\eta_1$. It is enough to prove $f\left(\eta_2\right)\leq f\left(\eta_1\right)$ which follows directly by simple algebraic manipulations. 
\end{proof}

\subsection{Proof of Theorem~\ref{thm:maingaelr}}\label{app:maingaelrproof}
Our analysis is based on the analysis in \cite{hazan} and brought here
for completeness.
First, we state the second-order bound for
the EG algorithm.
\begin{lemma}[simplified version of Lemma II.3 of \cite{clarkson}]\label{lm:hazanA1}
Let $\eta>0$, and let $\vec{c}_1,..,\vec{c}_t$ be an arbitrary sequence of
vectors in $\mathbb{R}^n$, with $\vec{c}_t\left[i\right]\geq-\frac{1}{\eta}$
for all $t$ and all $i\in\left[n\right]$. Define a sequence $\vec{z}_1,..,.\vec{z}_T$
by letting $\vec{z}_1=\vec{1}_n$ and for $t\geq1$,
\begin{equation*}
\vec{z}_{t+1}\left[i\right]=\vec{z}_t\left[i\right]\cdot\exp\left(-\eta\vec{c}_t\left[i\right]\right)\hspace*{10mm}i=1,..,n.
\end{equation*}
Then, for the vectors $\vec{p}_t=\frac{\vec{z}_t'}{\norm{\vec{z}'_t}{1}}$
we have
\begin{equation*}
\sum_{t=1}^m\vec{p}^T_t\vec{c}_t\leq\min_{i\in\left[n\right]}\sum_{t=1}^m\vec{c}_t\left[i\right]+\frac{\log
n}{\eta}+\eta\sum_{t=1}^m\vec{p}^T_t\vec{c}_t^2.
\end{equation*}
\end{lemma}

Now we examine the vectors $\vec{z}'=\left(\vec{z}_t^+,\vec{z}^-_t\right)\in\mathbb{R}^{2d}$
and $\vec{\bar{g}}'_t=\left(\vec{\bar{g}},-\vec{\bar{g}}\right)\in\mathbb{R}^{2d}$,
and setting $\vec{p}_t=\frac{\vec{z}_t'}{\norm{\vec{z}'_t}{1}}$. We have
the following lemma:

\begin{lemma}[Lemma 3.5 of \cite{hazan}]\label{lm:hazan35}
\begin{equation*}
\sum_{t=1}^m\vec{p}^T_t\vec{\bar{g}}'_t\leq\min_{i\in\left[2d\right]}\sum_{t=1}^m\vec{\bar{g}}_t'\left[i\right]+\frac{\log2d}{\eta}+\eta\sum_{t=1}^m\vec{p}^T_t\left(\vec{\bar{g}}'_t\right)^2.
\end{equation*}
\end{lemma}
Using this lemma, we establish
an expected risk bound with respect to the clipped linear functions $\vec{\bar{g}}^T_t\vec{w}$:

\begin{lemma}[Lemma 3.6 of \cite{hazan}]\label{lm:hazan36}
Assume that $\norm{\expectationDA{\vec{\widetilde{g}}^2_t}}{\infty}\leq G^2$
for all $t$, for some $G\geq0$. Then, for any $\norm{\vec{w}^*}{1}\leq B$,
 we have
\begin{equation*}
\expectationDA{\sum_{t=1}^m\vec{\bar{g}}^T_t\vec{w}_t}\leq\expectationDA{\sum_{t=1}^m\vec{\bar{g}}^T_t\vec{w}^*}+B\left(\frac{\log2d}{\eta}+\eta
G^2m\right).
\end{equation*}
\end{lemma}
For the proof of Lemma~\ref{lm:hazan37} we will need a simple lemma, that allows us to bound the deviation
of the expected value of a clipped random variable from that of the original
variable, in terms of its variance.
\begin{lemma}\label{lm:hazanA2}
Let $X$ be a random variable with $\left|\expectation{X}\right|\leq\frac{C}{2}$
for some $C>0$. Then for the clipped variable $\bar{X}=\clip\left(X,C\right)=\max\left\{\min\left\{X,C\right\},-C\right\}$
we have 
\begin{equation*} 
\left|\expectation{\bar{X}}-\expectation{X}\right|\leq2\frac{\mathrm{Var}\left[X\right]}{C}.
\end{equation*}
\end{lemma}

The next step is to relate the risk generated by the linear functions $\vec{\widetilde{g}}^T_t\vec{w}$,
to that generated by the clipped functions, $\vec{\bar{g}}^T_t\vec{w}$.

\begin{lemma}[A correction of Lemma 3.7 of \cite{hazan}]\label{lm:hazan37}
Assume that $\norm{\expectation{\vec{\widetilde{g}}^2_t}}{\infty}\leq G^2$
for all $t$, for some $G\geq0$. Then, for $0\leq\eta\leq\frac{1}{2G}$,
 we have
\begin{equation*}
\expectationDA{\sum_{t=1}^m\vec{\widetilde{g}}^T_t\left(\vec{w}_t-\vec{w}^*\right)}\leq\expectationDA{\sum_{t=1}^m\vec{\bar{g}}^T_t\left(\vec{w}_t-\vec{w}^*\right)}+4B\eta
G^2m.
\end{equation*}
\end{lemma}
Using these lemmas, we proceed to the proof of the theorem. First, from Lemma~\ref{lm:ridgemain},
as the GAERR and GAELR algorithm build the gradient estimator using the same
method, we have $\expectationA{\widetilde{\vec{g}}_t}=\vec{g_t}$.
From this follows that $\expectationA{\sum_{t=1}^m\vec{\widetilde{g}}^T_t\left(\vec{w}_t-\vec{w}^*\right)}=\expectationA{\sum_{t=1}^m\vec{g}^T_t\left(\vec{w}_t-\vec{w}^*\right)}$.
Combining this with Lemmas \ref{lm:hazan36} and \ref{lm:hazan37}, for $\eta\leq\frac{1}{2G}$,
we have

\begin{equation*}
\expectationDA{\sum_{t=1}^m\vec{g}^T_t\left(\vec{w}_t-\vec{w}^*\right)}\leq\frac{B\log2d}{\eta}+5B\eta
G^2m.
\end{equation*}
Proceeding as in the proof of Theorem~\ref{thm:maingaerr} finishes the proof
of Theorem~\ref{thm:maingaelr}.

\begin{proof}[Proof of Lemma~\ref{lm:hazanA1}]
Using the fact that $e^x\leq1+x+x^2$, for $x\leq1$, we have
\begin{equation*}
\norm{\vec{z}_{t+1}}{1}=
\sum_{i=1}^n\vec{z}_t\left[i\right]\cdot e^{-\eta\vec{c}_t\left[i\right]}\leq
\sum_{i=1}^n\vec{z}_t\left[i\right]\cdot \left(1-\eta\vec{c}_t\left[i\right]+\eta^2\vec{c}_t\left[i\right]^2\right)=
\norm{\vec{z}_t}{1}\cdot\left(1-\eta \vec{p}^T_t\vec{c}_t+\eta^2 \vec{p}^T_t\vec{c}_t^2\right),
\end{equation*}
and since $e^z\geq1+z$ for $z\in\mathbb{R}$, this implies by induction that
\begin{equation*}
\log{\norm{\vec{z}_{T+1}}{1}}
=\log n+\sum_{t=1}^T\log\left(1-\eta \vec{p}^T_t\vec{c}_t+\eta^2 \vec{p}^T_t\vec{c}_t^2\right)
\leq\log n-\eta\sum_{t=1}^T\vec{p}^T_t\vec{c}_t+\eta^2\sum_{t=1}^T \vec{p}^T_t\vec{c}_t^2.
\end{equation*}
On the other hand, we have
\begin{equation*}
\log{\norm{\vec{z}_{T+1}}{1}}=
\log\sum_{i=1}^n\prod_{t=1}^Te^{\eta\vec{c}_t\left[i\right]}\geq\\
\log\prod_{t=1}^Te^{\eta\vec{c}_t\left[i^*\right]}=\\
-\eta\sum_{t=1}^T\vec{c}_t\left[i^*\right].
\end{equation*}
Combining these two and rearranging, we obtain 
\begin{equation*}
\sum_{t=1}^m\vec{p}^T_t\vec{c}_t\leq\sum_{t=1}^m\vec{c}_t\left[i^*\right]+\frac{\log
n}{\eta}+\eta\sum_{t=1}^m\vec{p}^T_t\vec{c}_t^2
\end{equation*}
for any $i^*$, which completes the proof.
\end{proof}

\begin{proof}[Proof of Lemma~\ref{lm:hazan35}]
To see how Lemma~\ref{lm:hazan35} follows from Lemma~\ref{lm:hazanA1}, note that we can
write the update rule of the GAELR algorithm 
in the terms of the augmented vectors, $\vec{z}_t$ and $\vec{\bar{g}}'_t$
as follows
\begin{equation*}
\vec{z}_{t+1}\left[i\right]=\vec{z}_t\left[i\right]\cdot\exp\left(-\eta\vec{\bar{g}}'_t\left[i\right]\right)\hspace*{10mm}i=1,..,2d.
\end{equation*}
That is, $\vec{z}_{t+1}$ is obtained from $\vec{z}_t$ by a multiplicative
update based on the vector $\vec{\bar{g}}'_t$. Noticing that $\norm{\vec{\bar{g}}'_t}{\infty}=\norm{\vec{\bar{g}}_t}{\infty}\leq\frac{1}{\eta}$,
we see from Lemma~\ref{lm:hazanA1} that for any $i^*$,
\begin{equation*}
\sum_{t=1}^m\vec{p}^T_t\vec{\bar{g}}'_t\leq\sum_{t=1}^m\vec{\bar{g}}_t'\left[i^*\right]+\frac{\log2d}{\eta}+\eta\sum_{t=1}^m\vec{p}^T_t\left(\vec{\bar{g}}'_t\right)^2,
\end{equation*}
where $\vec{p}_t=\frac{\vec{z}_t'}{\norm{\vec{z}'_t}{1}}$, which gives the
lemma.
\end{proof}

\begin{proof}[Proof of Lemma~\ref{lm:hazan36}]
Notice that by our notation,
\begin{equation*}
\sum_{t=1}^m\vec{p}^T_t\vec{\bar{g}}'_t=
\sum_{t=1}^m\frac{\left(\vec{z}_t^+,\vec{z}_t^-\right)^T\left(\vec{\bar{g}}_t,-\vec{\bar{g}}_t\right)}{\norm{\vec{z}_t^+}{1}+\norm{\vec{z}_t^-}{1}}=
\frac{1}{B}\sum_{t=1}^m\vec{w}^T_t\vec{\bar{g}}_t
\end{equation*}
and
\begin{equation*}
\min_i\sum_{t=1}^m\vec{\bar{g}}_t'\left[i\right]=
\min_{\norm{\vec{w}}{1}\leq B}\frac{1}{B}\sum_{t=1}^m\vec{w}^T\vec{\bar{g}}_t\leq
\frac{1}{B}\sum_{t=1}^m{\vec{w}^*}^T\vec{\bar{g}}_t
\end{equation*}
for any $\vec{w}^*$ with $\norm{\vec{w}^*}{1}\leq B$. Plugging into the bound
of Lemma~\ref{lm:hazan35}, we get
\begin{equation*}
\sum_{t=1}^m\vec{\bar{g}}_t\left(\vec{w}_t-\vec{w}^*\right)\leq B\left(\frac{\log2d}{\eta}+\eta\sum_{t=1}^m\vec{p}^T_t\left(\vec{\bar{g}}'_t\right)^2\right).
\end{equation*}
Finally, taking the expectation with respect to the randomization of the
algorithm and the data distribution, and noticing that $\norm{\expectationDA{\left(\vec{\bar{g}}'_t\right)^2}}{\infty}\leq\norm{\expectationDA{\vec{\widetilde{g}}_t^2}}{\infty}\leq
G^2$, the proof is complete.
\end{proof}

\begin{proof}[Proof of Lemma~\ref{lm:hazanA2}]
As a first step, note that for $x>C$ we have $x-\expectation{X}\geq C/2$,
so that
\begin{equation*}
C\left(x-C\right)\leq2\left(x-\expectation{X}\right)\left(x-C\right)\leq2\left(x-\expectation{X}\right)^2. \end{equation*}
Hence, denoting by $\mu$ the probability measure of $X$, we obtain
\begin{align*}
\left|\expectation{\bar{X}}-\expectation{X}\right|
&\leq\int_{x<-C}\left(x+C\right)d\mu+\int_{x>C}\left(x-C\right)d\mu\\
&\leq\int_{x>C}\left(x-C\right)d\mu\\
&\leq\frac{2}{C}\int_{x>C}\left(x-\expectation{X}\right)^2d\mu\\
&\leq2\frac{\mathrm{Var}\left[X\right]}{C}.
\end{align*}
Similarly one can prove that $\expectation{\bar{X}}-\expectation{X}\geq-2\mathrm{Var}\left[X\right]/C$,
and the result follows.
\end{proof}

\begin{proof}[Proof of Lemma~\ref{lm:hazan37}]
Notice that $\norm{\expectationDA{\vec{\widetilde{g}}_t^2}}{\infty}\leq
G^2$ implies $\norm{\expectationDA{\vec{\widetilde{g}}_t}}{\infty}\leq G$
as  
\begin{equation*}
\norm{\expectationDA{\vec{\widetilde{g}}_t}}{\infty}^2=
\norm{\expectationDA{\vec{\widetilde{g}}_t}^2}{\infty}\leq
\norm{\expectationDA{\vec{\widetilde{g}}_t^2}}{\infty}.
\end{equation*}
Since $\vec{\bar{g}}\left[i\right]=\clip\left(\vec{\widetilde{g}}\left[i\right],1/\eta\right)$
and $\left|\expectationDA{\widetilde{\vec{g}}_t\left[i\right]}\right|\leq
G \leq 1/2\eta$ the above lemma implies that
\begin{equation*}
\left|\expectationDA{\vec{\bar{g}}_t\left[i\right]}-\expectationDA{\vec{\widetilde{g}}_t\left[i\right]}\right|\leq
2\eta\expectationDA{\vec{\widetilde{g}}_t\left[i\right]^{2}}\leq2\eta G^2
\end{equation*}
for all $i$, which means $\norm{\expectationDA{\widetilde{\vec{g}}_t-\bar{\vec{g}}_t}}{\infty}\leq2\eta
G^2$. Together with $\norm{\vec{w}_t-\vec{w}^*}{1}\leq2B$, this implies,
\begin{equation*}
\expectationDA{\left(\bar{\vec{g}}_t-\widetilde{\vec{g}}_t\right)^T\left(\vec{w}_t-\vec{w}^*\right)}\leq4\eta
G^2.
\end{equation*}
Summing over $t=1,..,m$, and taking the expectations, we obtain the lemma.
\end{proof}

\subsection{Proof of Lemma~\ref{lm:infinitynormbound}}\label{app:infinitynormboundproof}
From the definition of $\widetilde{\vec{x}}_{t}$ in equation~$\eqref{eq:ridgeestimator}$,
\begin{align*}
\norm{\expectationDA{\widetilde{\vec{x}}_{t}^2}}{\infty}
&=\norm{\expectationDA{\widetilde{\vec{x}}_{t}\left[i\right]^2}}{\infty}\\
&=\norm{\expectationDA{\left(\frac{1}{k}\sum_{r=1}^{k}\widetilde{\vec{x}}_{t,r}\left[i\right]\right)^2}}{\infty}\\
&=\norm{\frac{1}{k^{2}}\sum_{r=1}^{k}\expectationDA{\widetilde{\vec{x}}_{t,r}^2\left[i\right]} + \frac{1}{k^{2}}\sum_{r\neq s}^{k}\expectationDA{\widetilde{\vec{x}}_{t,r}\left[i\right]}^{2}}{\infty}.
\end{align*}
Since $\expectationDA{\widetilde{\vec{x}}_{t,r}\left[i\right]}=\expectationD{\vec{x}\left[i\right]}$,
$\widetilde{\vec{x}}_{t,r}\left[i\right]$ and $\widetilde{\vec{x}}_{t,s}\left[i\right]$ are independent of each other, and using the triangle inequality, we finally have
\begin{equation*}
\norm{\expectationDA{\widetilde{\vec{x}}_{t}^2}}{\infty}\leq
\max_i\frac{1}{k}\expectationDA{\widetilde{\vec{x}}_{t,r}^2\left[i\right]}+\frac{k-1}{k}\expectationD{\norm{\vec{x}}{\infty}}^2.
\end{equation*}

\subsection{Proof of Lemma~\ref{lm:lassooptimalqi}}\label{app:lassooptimalqiproof}
Let $C_i=\frac{\mathbb{E}_{D}\left[x_i^2\right]}{q_i}$. Note that $q_i=\frac{\expectationD{x_i^2}}{\sum_{j=1}^d\expectationD{x_j^2}}$ if, and only if, all $C_i$ are equal. Assume by contradiction that all $C_i$ are not equal, yet they still yield the minimal value for $\max_i\frac{1}{q_i}\expectationD{x_i^2}$. Let $I=\left\{i|C_i=max_jC_j\right\}$, and $i_0$ be an index for which $C_{i_0}<max_jC_j$, which exists, by our assumption. For $\Delta>0$, consider a new set of $q'_i$-s, such that $q'_{i_0}=q_{i_0}-\Delta$, and $q'_i=q_i+\frac{\Delta}{\left|I\right|}$ for $i\in I$. For a small enough $\Delta$, still $C'_{i_0}<max_jC'_j$. Note that this is still a valid assignment of probabilities because $\sum_{i=1}^dq'_i=1$ and all $q'_i>0$ for a small enough $\Delta$. However, $max_jC'_j$ is smaller than $max_jC_j$, in contradiction to the assumption. Therefore, all $C_i$ are equal and the minimal value is attained when $q_i=\frac{\mathbb{E}_{D}\left[x_i^2\right]}{\sum_{j=1}^d\mathbb{E}_{D}\left[x_j^2\right]}$.

\subsection{Proof of Lemma~\ref{lm:twonormboundphilasso}}\label{app:twonormboundphilassoproof}
Recalling $\left|y_{t}\right|\leq B$ and using the inequality $\left(a-b\right)^2\leq2\left(a^2+b^2\right)$, by a straightforward calculation
 we obtain:
\begin{align*}
\expectationDA{\widetilde{\phi_{t}}^2}
&=\expectationDA{\left(\frac{w_{t,j}}{p_j}\vec{x}_{t}\left[j_{t}\right]-y_{t}\right)^2}\\ &\leq2\expectationDA{\left(\frac{w_{t,j}}{p_{j}}\vec{x}_{t}\left[j_{t}\right]\right)^{2}+y_{t}^{2}}\\
&\leq2\sum_{j=1}^{d}\frac{1}{p_j}w_{t,j}^{2}\expectationD{\vec{x}_{j}^{2}}+2B^2\\
&\leq2\sum_{j=1}^{d}\frac{\norm{\vec{w}_t}{1}}{\left|w_{t,j}\right|}w_{t,j}^{2}+2B^2\\
&\leq2\norm{\vec{w}_t}{1}\sum_{j=1}^{d}\left|w_{t,j}\right|+2B^2\\
&\leq4B^2.
\end{align*}

\subsection{Proof of Lemma~\ref{lm:lassonotworse}}\label{app:lassonotworseproof}
Using the definition of $\norm{\widetilde{\vec{x}}_{t,r}}{2}^{2}$,
\begin{equation*}
\norm{\expectationDAphase{2}{\widetilde{\vec{x}}_{t,r}^2}}{\infty}=
\max_i\expectationDAphase{2}{\widetilde{\vec{x}}_{t,r}^2\left[i\right]}=
\max_i\frac{1}{q_i}\expectationD{x_i^{2}}=
\sum_{j=1}^{d}\left(A\left[j\right]+\frac{13}{6}\epsilon\right)\max_i\frac{\expectationD{x_{i}^{2}}}{A\left[i\right]+\frac{13}{6}\epsilon}.
\end{equation*}
Using equations \eqref{eq:smallbounds}, we have
\begin{align*}
\norm{\expectationDAphase{2}{\widetilde{\vec{x}}_{t,r}^2}}{\infty}
&\leq\sum_{j=1}^{d}\left(2\expectationD{x_{j}^{2}}+\frac{7}{6}\epsilon+\frac{13}{6}\epsilon\right)\max_i\frac{\expectationD{x_{i}^{2}}}{\frac{1}{2}\expectationD{x_{i}^{2}}-\frac{5}{3}\epsilon+\frac{13}{6}\epsilon}\\
&\leq4\sum_{j=1}^{d}\left(\expectationD{x_{j}^{2}}+\frac{5}{3}\epsilon\right)\max_i\frac{\expectationD{x_{i}^{2}}}{\expectationD{x_{i}^{2}}+\epsilon}\\
&\leq4\sum_{j=1}^{d}\left(\expectationD{x_{j}^{2}}+\frac{5}{3}\epsilon\right)\max_i\frac{\expectationD{x_{i}^{2}}}{\expectationD{x_{i}^{2}}}\\
&\leq4\norm{\expectationD{\vec{x}^2}}{1}+\frac{20}{3}d\epsilon.
\end{align*}
If $\epsilon=1$, as equations \eqref{eq:smallbounds} hold with probability $1$, this bound also holds with probability $1$. If $\epsilon\leq1$, this bound holds with probability $\geq 1-\delta$.

\end{document}